\documentclass[preprint,12pt]{elsarticle}
\usepackage{multirow}
\usepackage[table,xcdraw]{xcolor}
\usepackage{hyperref}
\usepackage{times}
\usepackage{helvet}
\usepackage{courier}
\usepackage{color}
\usepackage[leqno]{amsmath}
\usepackage{amssymb,amsthm}
\usepackage{graphicx}
\usepackage[lined,boxed,commentsnumbered, ruled]{algorithm2e}

\theoremstyle{plain}
\newtheorem{proposition}{Proposition}
\newtheorem{theorem}[proposition]{Theorem}
\newtheorem{lemma}[proposition]{Lemma}

\newtheorem{cor}[proposition]{Corollary}
\theoremstyle{definition}
\newtheorem{definition}[proposition]{Definition}
\newtheorem{example}[proposition]{Example}

\newcommand{\dist}{\textsf{dist}}
\newcommand{\rel}{\textbf{Rel}}
\newcommand{\circw}{\diamond}

\newcommand{\bp}{{\bf P}}

\newcommand{\bc}{{\bf C}}
\newcommand{\bdc}{{\bf DC}}
\newcommand{\bdr}{{\bf DR}}
\newcommand{\bec}{{\bf EC}}
\newcommand{\beq}{{\bf EQ}}
\newcommand{\bo}{{\bf O}}
\newcommand{\bpo}{{\bf PO}}
\newcommand{\bpp}{{\bf PP}}
\newcommand{\bppi}{{\bf PP}^{-1}}
\newcommand{\btpp}{{\bf TPP}}
\newcommand{\btppi}{{\bf TPP}^{-1}}
\newcommand{\bntpp}{{\bf NTPP}}
\newcommand{\bntppi}{{\bf NTPP}^{-1}}

\newcommand{\tpp}[2]{#1\, \btpp\, #2}

\newcommand{\p}[2]{#1\, \bp\, #2}

\newcommand{\pp}[2]{#1\, \bpp\, #2}

\newcommand{\ntpp}[2]{#1\, \bntpp\, #2}

\newcommand{\conn}[2]{#1\, \bc\, #2}
\newcommand{\po}[2]{#1\, \bpo\, #2}
\newcommand{\ec}[2]{#1\, \bec\, #2}
\newcommand{\eq}[2]{#1\, \beq\, #2}
\newcommand{\dc}[2]{#1\, \bdc\, #2}
\newcommand{\dr}[2]{#1\, \bdr\, #2}

\newcommand{\csp}{\textsf{CSP}}

\newcommand{\hornr}{\ensuremath{\widehat{\mathcal{H}}_8}}
\newcommand{\horn}{\ensuremath{\mathcal{H}}}

\newcommand{\clb}{\ensuremath{\widehat{\mathcal{B}}}}
\newcommand{\basic}{\ensuremath{\mathcal{B}}_5}

\newcommand{\qcm}{\mathcal{M}}
\newcommand{\ct}{\operatorname{CT}}

\makeatletter
\def\old@comma{,}
\catcode`\,=13
\def,{%
  \ifmmode%
    \old@comma\discretionary{}{}{}%
  \else%
    \old@comma%
  \fi%
}
\makeatother

\begin{document}

\begin{frontmatter}

\title{On Redundant Topological Constraints}

\author[uts]{Sanjiang Li\corref{cor1}}
\ead{sanjiang.li@uts.edu.au}
\author[uts]{Zhiguo Long}
\ead{zhiguo.long@student.uts.edu.au}
\author[baidu]{Weiming Liu}
\ead{liuweiming@baidu.com}
\author[melu]{Matt Duckham}
\ead{matt@duckham.org}
\author[melu]{Alan Both}
\ead{aboth@student.unimelb.edu.au}

\cortext[cor1]{Corresponding Author}

\address[uts]{Centre for Quantum Computation \& Intelligent Systems, University of Technology Sydney}
\address[baidu]{Baidu (China) Co., Ltd., Shanghai,  China}
\address[melu]{Department of Infrastructure Engineering, University of Melbourne}

\begin{abstract}
\noindent The Region Connection Calculus (RCC) \cite{RandellCC92} is a well-known calculus for representing part-whole and topological relations. It plays an important role in qualitative spatial reasoning, geographical information science, and ontology. The computational complexity of reasoning with {RCC5 and RCC8 (two fragments of RCC) as well as other qualitative spatial/temporal calculi} has been investigated in depth in the literature. Most of these works focus on the \emph{consistency} of qualitative constraint networks. In this paper, we consider the important problem of redundant {qualitative} constraints. For a set $\Gamma$ of {qualitative} constraints, we say a constraint $(x R y)$ in $\Gamma$ is \emph{redundant} if it {is entailed} by the rest of $\Gamma$. A \emph{prime subnetwork} of $\Gamma$ is a subset of $\Gamma$ which contains no redundant constraints and has the same solution set as $\Gamma$. It is natural to ask how to compute such a prime subnetwork, and when it is unique. 

In this paper, we show that this problem is in general intractable, but becomes tractable if $\Gamma$ is over a tractable subalgebra $\mathcal{S}$ of {a qualitative calculus}. {Furthermore, if} $\mathcal{S}$ is a subalgebra {of RCC5 or RCC8} in which weak composition distributes over nonempty intersections, then $\Gamma$ has a \emph{unique} prime subnetwork, which can be obtained {in cubic time} by removing all redundant constraints simultaneously  from $\Gamma$. As a byproduct, {we show that any path-consistent network over such a distributive subalgebra is weakly globally consistent and minimal. A thorough empirical analysis of the prime subnetwork upon real geographical data sets demonstrates the approach is able to identify significantly more redundant constraints than previously proposed algorithms, especially in constraint networks with larger proportions of partial overlap relations. } 
\end{abstract}

\end{frontmatter}

\begin{abstract}
The Region Connection Calculus (RCC) is a well-known calculus for representing part-whole and topological relations. It plays an important role in qualitative spatial reasoning, geographical information science, and ontology. The computational complexity of reasoning with RCC has been investigated in depth in the literature. Most of these works focus on the consistency of RCC constraint networks. In this paper, we consider the important problem of redundant RCC constraints. For a set $\Gamma$ of RCC constraints, we say a constraint $(x R y)$ in $\Gamma$ is \emph{redundant} if {it is entailed} by the rest of $\Gamma$. A \emph{prime subnetwork} of $\Gamma$ is a subset of $\Gamma$ which contains no redundant constraints but has the same solution set as $\Gamma$. It is natural to ask how to compute a prime subnetwork, and when it is unique. In this paper, we show that this problem is in general co-NP hard, but becomes tractable if $\Gamma$ is over a tractable subclass of RCC. If $\mathcal{S}$ is a tractable subclass in which weak composition distributes over nonempty intersections, then we can show that $\Gamma$ has a unique prime subnetwork, which is obtained by removing all redundant constraints from $\Gamma$. As a byproduct, we identify a sufficient condition for a path-consistent network being minimal.
\end{abstract}

\section{Introduction}

Qualitative spatial reasoning is a common subfield of artificial intelligence and geographical information science, and has applications ranging from natural language understanding \cite{Davis12}, robot navigation \cite{ShiJK10,Falomir12}, geographic information systems (GISs) \cite{EgenhoferM95}, sea navigation \cite{Wolter+08}, to high level interpretation of video data \cite{SridharCH11,CohnRS12}. 

{Typically, the qualitative approach represents spatial information by introducing a relation model on a domain of spatial entities, which could be points, line segments, rectangles, or arbitrary regions. In the literature, such a relation model is often called a \emph{qualitative calculus} \cite{LigozatR04}. In the past three decades, dozens of spatial (as well as temporal) qualitative calculi have been proposed in the literature (cf. \cite{CohnR08,RenzN07}). Among these,  Interval Algebra (IA) \cite{Allen83} and the RCC8 algebra \cite{RandellCC92} are widely known as the most influential qualitative calculi for representing qualitative temporal and, respectively, spatial information.  Other well-known qualitative calculi include Point Algebra (PA) \cite{VilainK86},  Cardinal Relation Algebra (CRA) \cite{Ligozat98}, Rectangle Algebra (RA) \cite{Guesgen89}, the RCC5 algebra \cite{RandellCC92}, etc.
}

{Using a qualitative calculus $\qcm$, we represent spatial or temporal information in terms of basic or non-basic relations in $\qcm$, and  formulate a spatial or temporal problem as a set of qualitative constraints (called a \emph{qualitative constraint network}). A qualitative constraint has the form $(x R y)$, which specifies that the two variables $x,y$ are related by the relation $R$. The \emph{consistency problem} is to decide whether a set of qualitative constraints can be satisfied simultaneously. The consistency problem has been investigated in depth for many qualitative calculi in the literature, e.g., \cite{VilainK86,Beek89,Ligozat98,NebelB95,Nebel95,RenzN97,Renz99,DavisGC99,WolterZ00,Liu+10,Kon11,LiuL11,SchockaertL12,LiLW13}.
}

In this paper, we consider the important problem of redundant {qualitative} constraints. Given a set $\Gamma$ of {qualitative} constraints, we say a constraint $(x R y)$ in $\Gamma$ is \emph{redundant} if {it is entailed} by the rest of $\Gamma$, i.e., removing $(x R y)$ from $\Gamma$ will not change the solution set of $\Gamma$. It is natural to ask {when a network contains redundant constraints} and how to get a non-redundant subset without changing the solution set. We call a subset of $\Gamma$ a \emph{prime subnetwork} of $\Gamma$ if it contains no redundant constraints and has the same solution set as $\Gamma$. 

{The redundancy problem (i.e., the problem of determining if a constraint is redundant in a network) was first considered by Egenhofer and Sharma \cite{EgenhoferS93} for topological constraints. They observed that a minimal set (i.e., a prime subnetwork) contains somewhere between $(n-1)$ and $(n^2-n)/2$ nontrivial relations, but did not provide efficient algorithms for deriving such a minimal set even for basic topological constraints. In a recent paper, Wallgr{\"u}n \cite{Wallgrun12} proposed two algorithms to approximately find the prime subnetwork. As observed in \cite{Wallgrun12}, and explored in more detail in Section 6, neither of these two algorithms is guaranteed to provide the optimal simplification.}

{The redundancy problem is also related to the minimal label problem (cf.  \cite{Montanari,ChandraP05,GereviniS11,LiuL12}). A qualitative constraint network $\Gamma$ is called \emph{minimal} if for each constraint $(x R y)$ in $\Gamma$, $R$  is the minimal (i.e., the \emph{strongest}) relation between $x,y$ that is entailed by $\Gamma$. Roughly speaking, the minimal network removes `redundant' or `unnecessary' \emph{relations} from each constraint, while the redundancy problem removes `redundant' or `unnecessary' \emph{constraints} from the constraint network.}

{We show in this paper that it is in general co-NP hard to determine if a constraint is redundant in a \emph{qualitative constraint network}. But if all constraints in $\Gamma$ are taken from a tractable subclass\footnote{Here a subclass $\mathcal{S}$ is \emph{tractable} if the consistency of any constraint network defined over $\mathcal{S}$ can be determined in polynomial time.} $\mathcal{S}$ then a prime subnetwork can be found in polynomial time. For example, if $\mathcal{S}$ is a tractable subclass of RCC5 or RCC8 that contains all basic relations, then we can find a prime subnetwork in $O(n^5)$ time. Furthermore, if $\mathcal{S}$ is a subalgebra of RCC5 or RCC8 in which weak composition distributes over nonempty intersections, then $\Gamma$ has a unique prime subnetwork, which is obtained by removing all redundant constraints from $\Gamma$. We also devise a cubic time algorithm for computing this unique prime subnetwork, which has the same time complexity as the two approximate algorithms of Wallgr{\"u}n \cite{Wallgrun12}.}

{As a byproduct, we identify an important class of subalgebras of qualitative calculi, which, called \emph{distributive subalgebras}, are subalgebras of qualitative calculi in which weak composition distributes over nonempty intersections. We show that any path-consistent network over such a distributive subalgebra is weakly globally consistent and minimal, where \emph{weakly global consistency} is a notion similar to but weaker than the well-known notion of global consistency (cf. Definition~\ref{dfn:minimal-network}). For RCC8, we identify two maximal distributive subalgebras which are not contained in any other distributive subalgebras, one contains 41 relations and the other contains 64. The 41 relations contained in the first subalgebra are exactly the convex RCC8 relations identified in \cite{ChandraP05}.}

{In this paper, we are mainly interested in topological constraints, as these are the most important kind of qualitative spatial information. A large part of our results can easily be transplanted to other qualitative calculi like PA, IA, CRA and RA. In particular, let $\qcm$ be one of PA, IA, CRA and RA and $\mathcal{S}$ a distributive subalgebra of $\qcm$ over which path-consistency implies consistency. Then we can show that any path-consistent network over $\qcm$ is globally consistent and minimal. For ease of presentation, we state and prove these results only for RCC5 and RCC8, but indicate in Table~\ref{tab:results} which result is applicable to which calculus.}

\subsection{Motivation}\label{sec:motivation}
{As} in the case of propositional logic formulas \cite{Liberatore05}, redundancy of {qualitative} constraints ``often leads to unnecessary computation, wasted storage, and may obscure the structure of the problem" \cite{BelovJLM12}.\footnote{It is worth noting that redundancy can also enhance propagation during computation (cf.~\cite{ChoiLS07}).} Finding a prime subnetwork can be useful in at least the following aspects: a) computing and storing the relationships between spatial objects and hence saving space for storage and communication; b) facilitating comparison (or measure the distance) between different constraint networks;  c)  unveiling the essential network structure of a network (e.g., being a tree or a graph with a bounded tree-width); and d) adjusting geometrical objects to meet topological constraints \cite{Wallgrun12}. 


{
To further motivate our discussion, we focus on one specific application to illustrate the application area a.~and briefly explain how redundancy checking or finding a prime subnetwork helps to solve the application areas b--d.}

Figure~\ref{fig:southampton} gives a small example of a set of spatial regions formed by the geographic ``footprints'' associated with placenames in the Southampton area of the UK. The footprints are derived from crowd-sourced data, formed from the convex hull of the sets of coordinate locations at which individuals used the placenames on social media (cf.~\cite{twaroch09.WEAVING,hollenstein10.JOSIS}). {Communicating and reasoning with the qualitative aspects of such data may require the storage and manipulation of large numbers of complex geometries with millions of {vertices} or large constraint networks with millions of relations. }

\begin{figure}[htb]
\centering
\hspace*{-8mm}
\includegraphics[width=1.15\columnwidth]{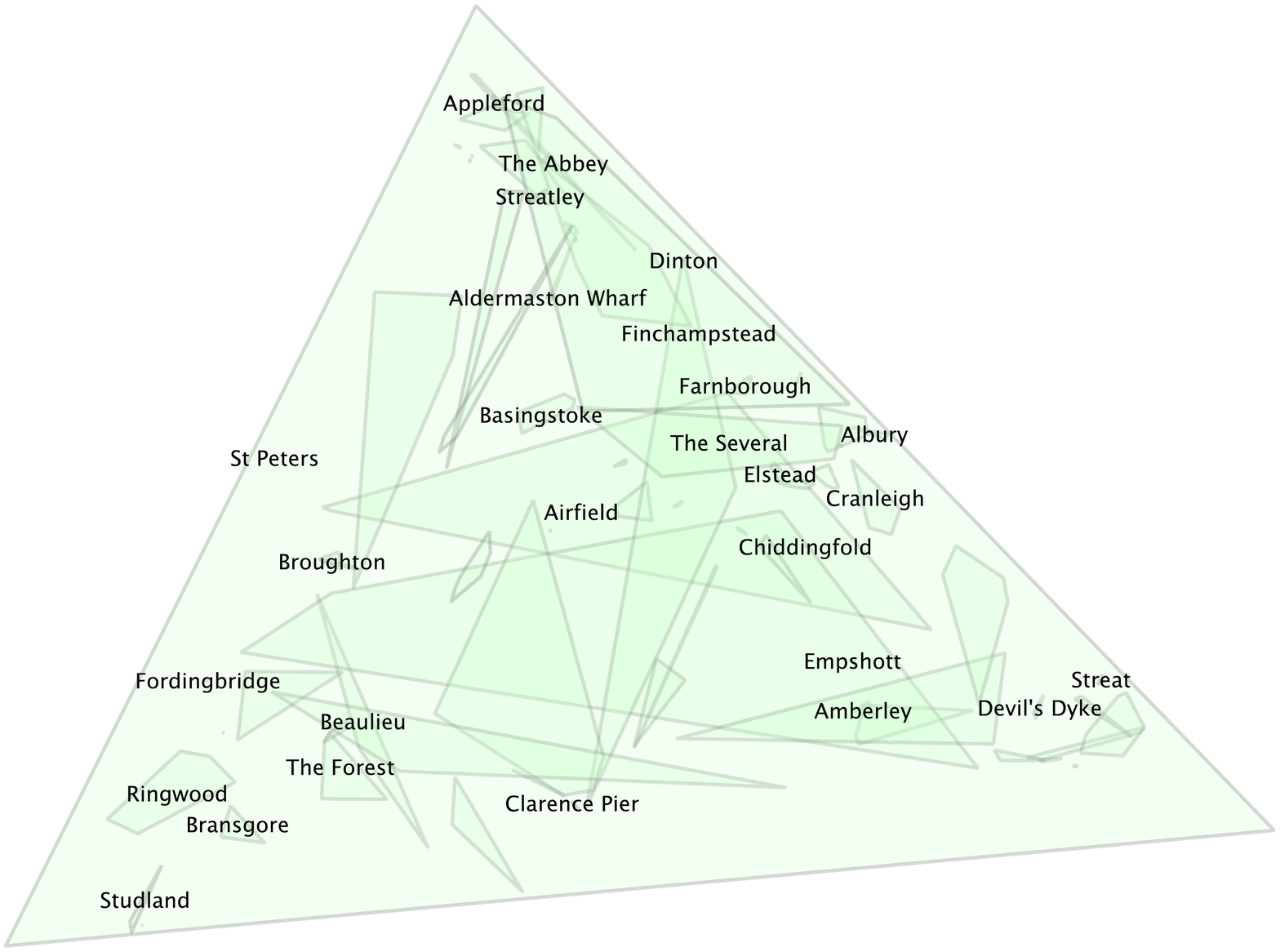}
\caption{Examples of crowd-sourced geographic placename ``footprints'' around Southampton, UK}\label{fig:southampton}
\end{figure}

{Even for} the small example in Figure \ref{fig:southampton}, the $84$ footprints then require $84*83/2=3486$ stored relations. The moderate-sized footprint data set from which Figure \ref{fig:southampton} is adapted contains a total of {$3443$} footprints {which leads} to a constraint network with {$5,925,403$} relations. {Similarly, a moderate-sized geographic data set of only {$1559$} statistical areas in Tasmania, explored further in later sections, contains in total {$3,093,551$} vertices.} In the case of both footprints and statistical areas, many of the relationships can be inferred, and {computing the prime subnetwork can potentially reduce the number of stored relationships to be approximately linear in the number of regions (i.e., average-case space complexity of $O(n)$), as opposed to linear in the number of relations (i.e., space complexity $\Theta(n^2)$)  (see Section \ref{sec:empirical_evaluation}).} In the case of the Southampton constraint network, $1324$ redundant relations lead to a prime subnetwork with only {$2162$} relations needing to be stored. For the full data set, {$5,604,200$} redundant relations lead to a prime subnetwork of just {$321,203$} relations (in contrast to the full constraint network of {almost} 6 million relations).

{
As for application area b., suppose $\Gamma,\Gamma'$ are two constraint networks over the same set of $n$ variables. The similarity of $\Gamma$ and $\Gamma'$ can be measured by computing the distance of each constraint $(x R y)$ in $\Gamma$ with the corresponding constraint $(x R' y)$ in $\Gamma'$ and sum them up (see e.g., \cite{CondottaKS08,WallgruenD10,LiL13}), i.e., 
\begin{align*}
\dist(\Gamma,\Gamma')=\sum\{\dist(R,R'): (x R y)\in \Gamma\ \mbox{and}\  (x R' y) \in \Gamma'\}.
\end{align*} 
Clearly, if $\Gamma$ and $\Gamma'$ are complete networks, we need $O(n^2)$ additions. This number, however, can be significantly reduced if we use prime subnetworks.  Let $\Gamma_{pr}$ and $\Gamma_{pr}'$ be, respectively, prime subnetworks of $\Gamma$ and $\Gamma'$. We define
\begin{align*}
\dist_{pr}(\Gamma,\Gamma')=\sum\{\dist(R,R'): (x R y)\in \Gamma_{pr}\ \mbox{or}\  (x R' y) \in \Gamma_{pr}'\}.
\end{align*} 
That is, the distance of $\Gamma$ and $\Gamma'$ is approximated by $\dist_{pr}(\Gamma,\Gamma')$, which only involves constraints in either $\Gamma_{pr}$ or $\Gamma'_{pr}$. If $\Gamma_{pr}$ and $\Gamma_{pr}'$ are sparse enough, i.e., they contain a small number of (non-redundant) constraints, this will significantly simplify the comparison of two constraint networks. 
}

{
In the case of application area c., a prime subnetwork unveils the essential network structure, or the skeleton, of a network, and the relation between a prime subnetwork and a constraint network is analogous to the relation between a spanning tree/forest \cite{Bollobas98} and a graph. Moreover, by the results of \cite{BodirskyW11} and \cite{HuangLR13}, we know it is tractable to determine the consistency of a constraint network with a bounded tree-width. Therefore, in general, checking the consistency of a prime subnetwork will be easier than checking the consistency of the network itself.
}

{
As for application area d., Wallgr\"{u}n \cite{Wallgrun12} proposed a method for exploiting qualitative spatial reasoning for topological adjustment of spatial data. To simplify the complexity of topological adjustment, he suggested replacing the original constraint network (say $\Gamma$) by an equivalent one (say $\Gamma'$) which has fewer redundant constraints. It is clear that the fewer constraints contained in $\Gamma'$ the better it is. A prime subnetwork is, roughly speaking, an optimal solution and contains fewest constraints. Therefore, replacing $\Gamma$ with a prime subnetwork will significantly simplify the complexity of topological adjustment.    
}

\vspace*{2mm}

The remainder of this paper is structured as follows. We first recall the {RCC5 and RCC8 constraint languages and introduce the notion of distributive subalgebras}  in Section 2, and then define the key notions of redundant constraint and prime subnetwork in Section 3. In Section 4 we show that consistent {RCC5 or RCC8} networks over {distributive subalgebras} have unique prime subnetworks. In Section 5 we compare our results with related works. {In Section 6 we present a detailed evaluation of a practical implementation of our algorithm, in comparison with the approximations proposed by Wallgr\"{u}n \cite{Wallgrun12}.} Section 7 concludes the paper and outlines future research. 

An extended abstract of this paper {was presented in KR-2014 as a short paper}.

\section{{RCC5 and RCC8} Constraint Languages}

{Suppose $U$ is a domain of spatial or temporal entities. Write $\rel(U)$ for the Boolean algebra of binary relations on $U$. A \emph{qualitative calculus} $\qcm$ on $U$ is defined as a finite Boolean subalgebra of $\rel(U)$ which contains the identity relation on $U$ as an atom and is closed under converse, i.e., $R$ is in $\qcm$ iff its converse $$R^{-1}=\{(a,b)\in U\times U: (b,a)\in R\}$$ is in $\qcm$. A relation $\alpha$ in a qualitative calculus $\qcm$ is \emph{basic} if it is an atom in $\qcm$. Well-known qualitative calculi include, among others, PA \cite{VilainK86}, IA \cite{Allen83}, CRA  \cite{Ligozat98}, RA \cite{Guesgen89}, and RCC5 and RCC8 \cite{RandellCC92}.  
}

{Since we are mainly interested in topological constraints, in this section, we only recall the RCC5 and RCC8 constraint languages and  refer the reader to for example \cite{RenzN07,CohnR08,LiLW13} for discussion of the other calculi. For convenience, we denote by RCC5/8  either RCC5 or RCC8.} 

\subsection{RCC5 and RCC8}
{The RCC5/8 constraint language is a fragment of the Region Connection Calculus (RCC) \cite{RandellCC92}, which is perhaps the most influential formalism for spatial relations in artificial intelligence. The RCC is a first order theory based on a binary connectedness relation and has canonical models defined over connected topological spaces \cite{stell2000bca,LiY03a}.}

{Let $X$ be a connected topological space and $U$ the set of nonempty regular closed sets of  $X$.} We call each element in $U$ a \emph{region}. Note that a region may have multiple connected components as well as holes. Write $\bp$ for the binary ``part-of" relation on $U$, i.e., $x\bp y$ if $x\subseteq y$.  Define 
\begin{align*}
\pp{x}{y} \equiv\ & \p{x}{y} \wedge \neg(\p{y}{x})\\
x\bo y \equiv\ & (\exists z)(\p{z}{x}\wedge\p{z}{y})\\
\dr{x}{y} \equiv\ & \neg(x \bo y) \\
\po{x}{y} \equiv\ & x\bo y \wedge \neg (\p{x}{y}) \wedge \neg(\p{y}{x})\\
\eq{x}{y} \equiv\ & \p{x}{y} \wedge \p{y}{x}
\end{align*}
Write $\bppi$ for the converse of $\bpp$. Then
\begin{align}\label{eq:rcc5}
\mathcal{B}_5=\{\bdr, \bpo, \beq, \bpp, \bppi\}
\end{align}
is a \emph{jointly exhaustive and pairwise disjoint} (JEPD) set of relations, i.e., for any two regions $a,b\in U$, $a,b$ is related by exactly one of the above five relations. We call the Boolean algebra generated by these five relations the RCC5 algebra, which consists of all relations that are unions of the five basic relations in \eqref{eq:rcc5}. For convenience, we denote a non-basic RCC5 relation $R$ as the subset of $\mathcal{B}_5$ it contains. For example, we write $\{\bdr,\bpo,\bpp\}$ for the relation $\bdr\cup\bpo\cup\bpp$, and write $\star_5$ for the universal relation $\{\bdr,\bpo,\bpp,\bppi,\beq\}$. 

RCC5 relations are in essence part-whole relations. We next introduce a topological relation model. For two regions $a,b$, we say $a$ is \emph{connected} to $b$, written $\conn{a}{b}$, if $a\cap b\not=\varnothing$. Using $\bc$ and $\bp$, the following topological relations can be defined \cite{RandellCC92}: 
\begin{align*}
\dc{x}{y} \equiv\ & \neg(x \bc y) \\
\ec{x}{y} \equiv\ & \conn{x}{y} \wedge \neg(x \bo y) \\
\tpp{x}{y} \equiv\ & \pp{x}{y} \wedge  (\exists z) (\ec{z}{x}\wedge \ec{z}{y})\\
\ntpp{x}{y} \equiv\ & \pp{x}{y}\wedge \neg(\tpp{x}{y})
\end{align*}
Write $\btppi$ and $\bntppi$ for the converses of $\btpp$ and $\bntpp$. Then 
\begin{align}\label{eq:rcc8}
\mathcal{B}_8 &=\{\bdc,\bec, \bpo, \beq,
\btpp, \bntpp, \btppi,\bntppi\}
\end{align}
is a JEPD set of relations. We call the Boolean algebra generated by these eight relations the RCC8 algebra, which consists of all relations that are unions of the eight basic relations in \eqref{eq:rcc8}. For convenience, we write $\star_8$ for the universal relation consisting of all basic relations in $\mathcal{B}_8$.

\subsection{Weak Composition Table}

{While PA, IA, CRA and RA are all closed under composition, the} composition of two basic RCC5/8 relations is not necessarily a relation in RCC5/8 \cite{DuntschWM01,LiY03a}. For example, the composition of $\bdr$ and itself is not an RCC5 relation. This is because, {for example}, $\bpo$ intersects with, but is not contained in, $\bdr\circ\bdr$, where $\circ$ denotes the relational composition operator. In fact, there are three regions $a,b,c$ such that $a\bpo c$ and $\dr{a}{b}, \dr{b}{c}$. This shows that $\bpo\cap \bdr\circ\bdr$ is nonempty. Let $d,e$ be two regions such that $\po{d}{e} $ and $d\cup e =\mathbb{R}^2$. Clearly, there is no region $f$ such that $\dr{d}{f}$ and $\dr{f}{e}$ hold simultaneously. Therefore $\bpo$ is not contained in $\bdr\circ\bdr$.

For two RCC5/8 relations $ R $ and $ S $, we call the smallest relation in RCC5/8 that contains $ R \circ  S $ the \emph{weak composition} of $ R $ and $ S $, written $ R \diamond S $ \cite{DuntschWM01,LiY03a}. 

The weak  compositions of RCC5 and RCC8 basic relations are summarised in, respectively, Table~\ref{tab:rcc5ct} and Table~\ref{tab:RCC8} (from  \cite{RandellCC92}). For each pair of RCC5/8 basic relations $(\alpha,\beta)$, the table cell corresponding to $(\alpha,\beta)$ contains all basic relations that are contained in $\alpha\diamond\beta$. In fact, suppose $\alpha,\beta,\gamma$ are three basic RCC5/8 relations. Then we have
\begin{align}\label{eq:wc}
\gamma\in \alpha \circw \beta \Leftrightarrow \gamma \cap {(\alpha \circ \beta)} \not=\varnothing.
\end{align}
The weak composition of two (non-basic) RCC5/8 relations $R$ and $S$ can be computed as follows: 
\begin{align*}
R\circw S &=\bigcup\{\alpha\circw \beta: {\alpha\in R, \beta\in S}\}.
\end{align*}
Given $(x R y)$ and $(y S z)$, by definition, we have $(x R\circw S z)$, i.e., $\{(x R y), (y S z)\}$ entails $(x R\circw S z)$.
 
\begin{table*}[ht]
\centering
\scalebox{0.65}{
\begin{tabular} {c|ccccc}
$\diamond$  & $\bdr$ &$\bpo$ & $\bpp$ & $\bppi$ & $\beq$ \\ \hline
$\bdr$ & $\bdr$,$\bpo$,$\bpp$,$\bppi$,$\beq$ &  $\bdr$,$\bpo$,$\bpp$& $\bdr$,$\bpo$,$\bpp$ & $\bdr$ & $\bdr$ \\
$\bpo$ &   $\bdr$,$\bpo$,$\bppi$& $\bdr$,$\bpo$,$\bpp$,$\bppi$,$\beq$ & $\bpo$,$\bpp$& $\bdr$,$\bpo$,$\bppi$& $\bpo$ \\
$\bpp$ & $\bdr$ & $\bdr$,$\bpo$,$\bpp$ & $\bpp$ & $\bdr$,$\bpo$,$\bpp$,$\bppi$,$\beq$ & $\bpp$ \\
$\bppi$ & $\bdr$,$\bpo$,$\bppi$ & $\bpo$,$\bppi$ & $\bpo$,$\bpp$,$\bppi$,$\beq$ & $\bppi$  &$\bppi$ \\
$\beq$ & $\bdr$ &$\bpo$ & $\bpp$ & $\bppi$ & $\beq$
\end{tabular}
}
\caption{ Composition table for {\rm RCC5} relations} \label{tab:rcc5ct}
\end{table*}

 \begin{table*}[t]
 \vskip 3mm \centering
 {
  \scalebox{0.6}{
 \begin{tabular}{c|c|c|c|c|c|c|c|c}
 $\diamond$&\bdc&\bec&\bpo&\btpp&\bntpp&$\btppi$&$\bntppi$&\beq\\
 &&&&&&&&\\
 \hline &\bdc,\bec,\bpo   &\bdc,\bec &\bdc,\bec&\bdc,\bec&\bdc,\bec&  &&\\
      \bdc&\btpp,\bntpp   &\bpo    &\bpo   &\bpo   &\bpo   &\bdc&\bdc&\bdc\\
        &$\btppi$,\beq     &\btpp   &\btpp  &\btpp  &\btpp  &  &&\\
        &$\bntppi$      &\bntpp  &\bntpp &\bntpp &\bntpp &  &&\\
 \hline &\bdc,\bec,\bpo&\bdc,\bec,\bpo&\bdc,\bec,\bpo&\bec,\bpo&\bpo  &\bdc&&\\
      \bec&$\btppi$    &\beq,\btpp   &\btpp     &\btpp  &\btpp &\bec&\bdc&\bec\\
        &$\bntppi$   &$\btppi$    &\bntpp    &\bntpp &\bntpp&  &&\\
 \hline &\bdc,\bec,\bpo &\bdc,\bec,\bpo&\bdc,\bec,\bpo&\bpo&\bpo&\bdc,\bec,\bpo&\bdc,\bec,\bpo&\\
      \bpo&$\btppi$&$\btppi$&\btpp,$\btppi$,\beq&\btpp&\btpp&$\btppi$&$\btppi$&\bpo\\
        &$\bntppi$&$\bntppi$&\bntpp,$\bntppi$&\bntpp&\bntpp&$\bntppi$&$\bntppi$&\\
 \hline &  &\bdc&\bdc,\bec &\btpp &    &\bdc,\bec,\bpo  &\bdc,\bec,\bpo&\\
     \btpp&\bdc&\bec&\bpo,\btpp&\bntpp&\bntpp&\beq,\btpp     &$\btppi$&\btpp\\
        &  &  &\bntpp  &    &    &$\btppi$      &$\bntppi$&\\
 \hline &  &  &\bdc,\bec&    &    &\bdc,\bec&\bdc,\bec,\bpo&\\
    \bntpp&\bdc&\bdc&\bpo   &\bntpp&\bntpp&\bpo   &\btpp,$\btppi$&\bntpp\\
        &  &  &\btpp  &    &    &\btpp  &\bntpp,\beq&\\
        &  &  &\bntpp &    &    &\bntpp &$\bntppi$&\\
 \hline &\bdc,\bec,\bpo&\bec,\bpo&\bpo   &\bpo,\beq&\bpo  &$\btppi$ &&\\
    $\btppi$&$\btppi$    &$\btppi$ &$\btppi$ &\btpp &\btpp &     &$\bntppi$&$\btppi$\\
        &$\bntppi$   &$\bntppi$&$\bntppi$&$\btppi$&\bntpp&$\bntppi$&&\\
 \hline &\bdc,\bec,\bpo&\bpo   &\bpo   &\bpo   &\bpo,\btpp,\beq & & &\\
   $\bntppi$&$\btppi$    &$\btppi$ &$\btppi$ &$\btppi$ &\bntpp,$\btppi$&$\bntppi$&$\bntppi$&$\bntppi$\\        &$\bntppi$   &$\bntppi$&$\bntppi$&$\bntppi$&$\bntppi$    &  & &\\
 \hline \beq&\bdc&\bec&\bpo&\btpp&\bntpp&$\btppi$&$\bntppi$&\beq
 \end{tabular}
 }
}
 \caption{Composition table for {\rm RCC8} relations}\label{tab:RCC8}
 \end{table*}

From the RCC5 composition table, {the following result is clear.}

\begin{lemma}\label{lemma:po/dr/composition}
For any nonempty RCC5 relation $R$, we have
\begin{align*}
\bpo\in {\bpo\circw R},\ \bpo\in {R \circw \bpo},\ \mbox{and}\  \bdr\in {\bdr\circw R},\ \bdr\in {R\circw\bdr}.
\end{align*}
\end{lemma}
The following result will be used later.
\begin{proposition} [{from \cite{Duntsch05}}]
\label{prop:ra}
With the weak composition operation $\circw${, the converse operation $^{-1}$, and the identity relation \beq,} RCC5 and RCC8 are relation algebras. {In particular, the weak composition operation $\circw$ is associative.}  
Moreover, for RCC5/8 relations $R,S,T$, we have the following cycle law
\begin{align}\label{eq:cycle}
(R\circw S)\cap T\not=\varnothing  \Leftrightarrow (R^{-1} \circw T)\cap S\not=\varnothing &\Leftrightarrow (T\circw S^{-1})\cap R\not=\varnothing.
\end{align}
\end{proposition}
{Figure~\ref{fig:cycle} gives an illustration of the cycle law.} 
\begin{figure}[h]
\centering
\begin{tabular}{c}
\includegraphics[width=.8\columnwidth]{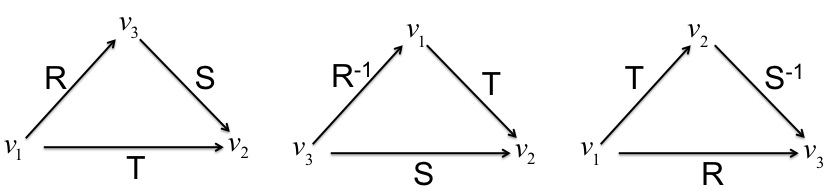} \\
\end{tabular}
\caption{Illustration of the cycle law.}\label{fig:cycle}
\end{figure}
{In the following, we assume $\circw$ takes precedence over $\cap$.}
\subsection{{Qualitative} Constraint Network}

{Let $\qcm$ be a qualitative calculus with domain $U$. A qualitative constraint over $\qcm$ has the form $(x  R  y)$, where $x,y$ are variables taking values from $U$ and $ R $ is a relation (not necessarily basic) in $\qcm$. Given a set $\Gamma$ of qualitative constraints over variables $V=\{v_1,v_2,...,v_n\}$ and an assignment $\sigma: V \to U$, we say $\sigma$ is a \emph{solution} of $\Gamma$  if $(\sigma(v_i),\sigma(v_j))$ satisfies the constraints in $\Gamma$ that relate $v_i$ to $v_j$ for any $1\leq i,j\leq n$. We say $\Gamma$ is \emph{consistent} or \emph{satisfiable} if it has a solution.} 

Without loss of generality, we assume
\begin{itemize}
\item for each pair of variables $v_i,v_j$, there is at most one constraint in $\Gamma$ that relates $v_i$ to $v_j$; 
\item for each pair of variables $v_i,v_j$, if there is no constraint in $\Gamma$ that relates $v_i$ to $v_j$, we say $v_i$ is related to $v_j$ by $\star$, the universal relation {in $\qcm$};
\item for each pair of variables $v_i,v_j$, the constraint in $\Gamma$ that relates $v_i$ to $v_j$ is the converse of the constraint that relates $v_j$ to $v_i$;
\item for each variable $v_i$, the constraint in $\Gamma$ that relates $v_i$ to itself is the identity relation {(e.g., \beq\  in RCC5/8)}.
\end{itemize}
In this sense, we call $\Gamma$ a \emph{network} of constraints, and denote by for example $R_{ij}$ the constraint that relates $v_i$ to $v_j$.  Let $\Gamma=\{v_i R_{ij} v_j{:\ 1\leq i,j\leq n} \}$ and $\Gamma'=\{v_i R'_{ij} v_j: {1\leq i,j\leq n} \}$ be two constraint networks {over $\qcm$}. We say $\Gamma$ and $\Gamma'$ are \emph{equivalent} if they have the same set of solutions; and say $\Gamma$ \emph{refines} $\Gamma'$ if $R_{ij}\subseteq R'_{ij}$ for all $(i,j)$. We say {a constraint network} $\Gamma$ is a \emph{basic} network if each constraint is either  a basic relation or the universe relation; and say a basic network is \emph{complete} if there are no universal relations. {In this paper, we also call every complete basic network that refines $\Gamma$ a \emph{scenario} of $\Gamma$.}

{
Suppose $\mathcal{S}$ is a subclass of $\qcm$.  We say a constraint network $\Gamma=\{v_i R_{ij} v_j: {1\leq i,j\leq n} \}$ is over $\mathcal{S}$ if $R_{ij}\in \mathcal{S}$ for every pair of variables $v_i,v_j$.  The consistency problem over $\mathcal{S}$, written as $\csp(\mathcal{S})$, is the decision problem of the consistency of an arbitrary constraint network over $\mathcal{S}$. The consistency problem over {PA}  (i.e., $\csp$(PA)) is in P \cite{VilainK86,Beek89} and the consistency problems over IA, CRA, RA and RCC5/8 are NP-complete \cite{NebelB95,Ligozat98,BalbianiCC99,RenzN97}. We say $\mathcal{S}$ is a tractable subclass of $\qcm$ if $\csp(\mathcal{S})$ is tractable. It is well-known that these calculi all have large tractable subclasses, in particular, RCC8 has three maximal tractable subclasses that contain all basic relations \cite{RenzN97,Renz99}  and RCC5 has only one \cite{Jonsson97}.
}

The consistency of {a qualitative constraint network} can be approximately determined by a local consistency algorithm. We say a network $\Gamma=\{v_i R_{ij}v_j: { 1\leq i,j\leq n} \}$ is \emph{path-consistent}\footnote{{For PA, IA, CRA and RA, since weak composition is composition, this definition of path-consistency is equivalent to that for finite constraint satisfaction problems \cite{Montanari}; for RCC5/8, the two definitions are different mainly in the use of weak composition instead of composition.}} if for every $1\leq i,j,k\leq n$, we have\footnote{{Recall we have assumed that $R_{ji}$ is the converse of $R_{ij}$ for each pair of variables $v_i,v_j$.}}
\begin{align*}
{\varnothing\not=}R_{ij} \subseteq R_{ik}\diamond R_{kj}.
\end{align*}
In general, path-consistency can be enforced by calling the following rule until an empty constraint occurs (then $\Gamma$ is inconsistent) or the network becomes stable\footnote{{Under the assumption that initially we have $R_{ji}=R_{ij}^{-1}$ for every $i\not=j$, we do not need to call updating rules like $R_{ji} \leftarrow R_{ij}^{-1}$, as this can be achieved by calling $R_{ji} \leftarrow (R_{jk}\diamond R_{ki}) \cap R_{ji}$ after $R_{ij} \leftarrow (R_{ik}\diamond R_{kj}) \cap R_{ij}$ is called. This will simplify the discussion in, for example, the proof of Lemma~\ref{lem:int-paths}.}}
\begin{align*}
R_{ij} \leftarrow (R_{ik}\diamond R_{kj}) \cap R_{ij},
\end{align*}
where $1\leq i,j,k\leq n$ are arbitrary. A cubic time algorithm, henceforth called the \emph{path-consistency algorithm} or PCA, has been devised to enforce path-consistency. For any {qualitative constraint network} $\Gamma$, the PCA either detects inconsistency of $\Gamma$ or returns a path-consistent network, written $\Gamma_p$, which is equivalent to $\Gamma$ and also known as the \emph{algebraic closure} or \emph{a-closure} of $\Gamma$ \cite{LigozatR04}. It is easy to see that in this case $\Gamma_p$ refines $\Gamma$, i.e., we have $S_{ij} \subseteq R_{ij}$ for each constraint $(v_i S_{ij} v_j)$ in $\Gamma_p$.

{For RCC5/8 constraint networks, we have}
 \begin{proposition} [{from \cite{Renz99}}]\label{prop:horn5}
 Let $\mathcal{S}$ be a tractable subclass of RCC5/8 which contains all basic relations. An RCC5/8 network $\Gamma$ over $\mathcal{S}$ is consistent if applying PCA to $\Gamma$ does not detect inconsistency.
 \end{proposition}
 
In particular, we have 
 \begin{proposition} [{from \cite{Nebel95}}]\label{prop:basic}
 A basic RCC5/8 network $\Gamma$ is consistent if it is path-consistent.
 \end{proposition}

{Consistency is closely related to the notions of minimal network (cf.  \cite{ChandraP05,GereviniS11,LiuL12}) and global consistency.}
\begin{definition} \label{dfn:minimal-network}
{
Let $\qcm$ be a qualitative calculus with domain $U$. Suppose $\Gamma=\{v_i T_{ij} v_j: {1\leq i,j\leq n}\}$ is a qualitative constraint network over $\qcm$ and $V=\{v_1,...,v_n\}$. For a pair of variables $v_i,v_j \in V$ ($i\not=j$) and a basic relation $\alpha$ in $T_{ij}$, we say $\alpha$ is \emph{feasible} if there exists a solution $(a_1,a_2,\ldots, a_n)$ in $U$ of $\Gamma$ such that $(a_i,a_j)$ is an instance of $\alpha$. We say $\Gamma$ is \emph{minimal} if $\alpha$ is feasible for every pair of variables $v_i,v_j$ ($i\not=j$) and every basic relation $\alpha$ in $T_{ij}$.}

{
We say $\Gamma$ is \emph{weakly globally consistent} (\emph{globally consistent}, respectively) if any consistent scenario (solution, respectively) of $\Gamma{\downarrow}_{V'}$ can be extended to a consistent scenario (solution, respectively) of $\Gamma$, where $V'$ is any nonempty subset of $V$ and $\Gamma{\downarrow}_{V'}$ is the restriction of $\Gamma$ to $V'$.}   
\end{definition}
{
The notion of weakly global consistency is weaker than the notion of global consistency. The latter requires that every partial solution can be extended to obtain a global solution, which is too strong for even complete basic RCC5/8 networks.  But the two notions are equivalent for PA, IA, CRA and RA as consistent basic networks over these calculi are all globally consistent.}

{
While every consistent RCC5/8 constraint network has a unique minimal network, it is in general NP-hard to compute it \cite{LiuL12}. The following result shows that every weakly globally consistent network is also minimal.}
{
\begin{proposition}\label{wgcISmin}
 Let $\qcm$ be a qualitative calculus with domain $U$. Suppose $\Gamma=\{v_i T_{ij} v_j: {1\leq i,j\leq n}\}$ is a qualitative constraint network over $\qcm$. If $\Gamma$ is weakly globally consistent, then it is minimal.
 \end{proposition}
 \begin{proof}
 For every pair of variables $v_i,v_j$ ($i\not=j$) and every basic relation $\alpha$ in $T_{ij}$, it is clear that $\{v_i\alpha v_j\}$ is a consistent scenario of $\Gamma{\downarrow}_{\{v_i,v_j\}}$. Because $\Gamma$ is weakly globally consistent, we can extend this to a consistent scenario of $\Gamma$. In other words, there exists a solution $(a_1,a_2,\ldots, a_n)$ of $\Gamma$ in $U$ such that $(a_i,a_j)$ is an instance of $\alpha$. This shows that $\Gamma$ is minimal.
 \end{proof}
}

In what follows, we write $\Gamma_m$ for the minimal network of $\Gamma$, and  $\Gamma_p$ for the a-closure of $\Gamma$. 

\subsection{Distributive Subalgebra}
As mentioned before, RCC5 has a unique maximal tractable subclass which contains all basic relations \cite{RenzN97,Jonsson97}. This subclass, written $\horn_5$, contains all RCC5 relations except 
\begin{align*}
\{\bpp,\bppi\}, &\{\bpp,\bppi,\beq\}, \{\bdr,\bpp,\bppi\}, 
\{\bdr,\bpp,\bppi,\beq\}.
\end{align*}
Write $\clb_5$ for the closure of $\basic$ under converse, intersection, and weak composition  in RCC5. Then $\clb_5$ contains all basic relations as well as 
\begin{align*}
&\{\bpo,\bpp\},\{\bpo,\bppi\},\{\bpo,\bpp,\bppi,\beq\},\\
&\{\bdr,\bpo,\bpp\},\{\bdr,\bpo,\bppi\},\{\bdr,\bpo\},\star_5,
\end{align*}
where $\star_5=\{\bdr,\bpo,\bpp,\bppi,\beq\}$. It is interesting to note that in $\clb_5$ the weak composition operation is \emph{distributive} over nonempty intersections in the following sense.
\begin{lemma}\label{lemma:distri}
Let $R,S,T$ be three relations in $\clb_5$. Suppose $S\cap T$ is nonempty. Then we have 
\begin{align*}
R\circw(S\cap T) = R\circw S \cap R\circw T\quad\mbox{and}\quad (S\cap T)\circw R = S\circw R \cap T\circw R.
\end{align*}
\end{lemma}
{We note the requirement that $S\cap T$ is nonempty is necessary, as we have for example $\{\bdr\}\circw \{\bdr \}\cap \{\bdr\}\circw \{\bpo\} =\{\bdr,\bpo,\bpp\}\not=\varnothing$ while $\{\bdr\}\circw(\{\bdr\} \cap \{\bpo\})=\{\bdr\}\circw \varnothing=\varnothing$.}

In what follows, we call such a subclass a distributive subalgebra. Formally, we have 
\begin{definition} \label{dfn:distributive}
{Let $\qcm$ be a qualitative calculus. A subclass $\mathcal{S}$ of $\qcm$ is called} a \emph{distributive subalgebra}  if 
\begin{itemize}
\item $\mathcal{S}$ contains all basic relations; and
\item $\mathcal{S}$ is closed under converse, weak composition, and intersection; and
\item weak composition distributes over nonempty intersections of relations in $\mathcal{S}$.
\end{itemize}
\end{definition}
{Write $\widehat{\mathcal{B}}_l$ for the closure of $\mathcal{B}_l$ in RCC$l$ $(l=5,8)$ under converse, weak composition, and intersection. It is straightforward to check that both $\widehat{\mathcal{B}}_5$ and $\widehat{\mathcal{B}}_8$ are distributive subalgebras. This shows that the above definition is well-defined and  every distributive subalgebra of RCC$l$ contains $\widehat{\mathcal{B}}_l$ as a subclass. }

{We say a distributive subalgebra $\mathcal{S}$ is \emph{maximal} if there is no other distributive subalgebra that properly contains $\mathcal{S}$. To find all maximal distributive subalgebras of RCC5 and RCC8, we start with $\widehat{\mathcal{B}}_l$ and then try to add other relations to this subalgebra to get larger distributive subalgebras. It turns out that both RCC5 and RCC8 have only two maximal distributive subalgebras. In~\ref{app:dis} we list all relations contained in these subalgebras, and explain how we find these subalgebras and why there are no other maximal distributive subalgebras.
}

{The next lemma summarises one useful property of distributive subalgebras.}  

\begin{lemma}\label{lem:3x}
Let $\mathcal{S}$ be a distributive subalgebra of RCC5/8. Suppose $R,S,T$ are three relations in $\mathcal{S}$. Then $R\cap S\cap T=\varnothing$ iff $R\cap S=\varnothing$, or $R\cap T=\varnothing$, or $S\cap T=\varnothing$.
\end{lemma}
\begin{proof}
{We only need to show the ``only if" part.} 

{
For two RCC5/8 relations $P,Q$, we first note that $P\cap Q\not=\varnothing$ iff $\beq\in Q^{-1}\circw P$. In fact, from $P\cap Q\not=\varnothing$, we know there exist two regions $a,b$ such that $(a,b)\in P\cap Q$. This implies that $(b,b)\in Q^{-1}\circ P$ as $(b,a)\in Q^{-1}$ and $(a,b)\in P$. Hence $\beq\cap Q^{-1}\circ P$ is nonempty and, by the definition of weak composition and \eqref{eq:wc}, $\beq\in Q^{-1} \circw P$. On the other hand, if $\beq\in Q^{-1}\circw P$, then $\beq\cap Q^{-1}\circ P$ is nonempty. This implies that there exist two regions $a,b$ such that $(b,a)\in Q^{-1}$ and $(a,b)\in P$. Thus $(a,b)\in P\cap Q$ and, hence, $P\cap Q\not=\varnothing$. 
}

{
Suppose $R\cap S\cap T$ is empty but $R\cap S, R\cap T$ and $S\cap T$ are all nonempty. By the above property, we have $\beq\in T^{-1}\circw R$ and $\beq\in T^{-1}\circw S$. Because $R,S,T$ are relations in the distributive subalgebra $\mathcal{S}$ and $R\cap S\not=\varnothing$, we know $$\beq\in (T^{-1} \circw R) \cap (T^{-1} \circw S) = T^{-1} \circw (R\cap S).$$
Thus $T^{-1} \circw (R\cap S)\not=\varnothing$ and, hence, $R\cap S\cap T\not=\varnothing$ . A contradiction.
}
\end{proof}
{The above result does not hold in general. For example, consider the RCC5 relations  $R=\{\bpo,\bpp\}$, $S=\{\bdr,\bpp\}$, $T=\{\bdr,\bpo,\bppi\}$. $R,S,T$ are all in $\horn_5$ but $S$ is not in any distributive subalgebra of RCC5. We have $R\cap S\cap T=\varnothing$ but $R\cap S=\{\bpp\}$, $R\cap T=\{\bpo\}$, and $S\cap T=\{\bdr\}$ are all nonempty.
} 

\vspace*{2mm}

It is worth noting that each distributive subalgebra of RCC5 is contained in $\horn_5$, the maximal tractable subclass of RCC5  identified in \cite{RenzN97,Jonsson97}, and each distributive subalgebra of RCC8 is contained in $\hornr$, one of the three maximal subclasses of RCC8  identified in \cite{Renz99}. In particular, by Proposition~\ref{prop:horn5}, we have 
\begin{cor}\label{cor:distr-pc}
Let $\mathcal{S}$ be a {distributive} subalgebra of RCC5/8. Then every path-consistent network over $\mathcal{S}$ is consistent.
\end{cor} 
\section{Redundant Constraint and Prime Subnetwork}

In this section we first give definitions of redundant constraints and prime subnetworks and then discuss how to find a prime subnetwork in general. 

\begin{definition}
{Let $\qcm$ be a qualitative calculus with domain $U$.} Suppose $\Gamma$ is {a qualitative constraint network} over variables $V=\{v_1,...,v_n\}$. We say $\Gamma$ \emph{entails} a constraint $(v_i R  v_j)$, written $\Gamma\models (v_i R  v_j)$, if for every solution $\{a_1,...,a_n\}$ of $\Gamma$ in $U$ we have $(a_i,a_j)\in  R $.  A constraint $(v_i R  v_j)$ in $\Gamma$ is \emph{redundant} if $\Gamma\setminus\{(v_i R  v_j)\}$ entails $(v_i R  v_j)$.
We say  $\Gamma$ is \emph{reducible} if it has a redundant constraint, and say $\Gamma$ is \emph{irreducible} {or \emph{prime}} if otherwise. We say a subset $\Gamma'$  of $\Gamma$ is a \emph{prime} subnetwork of $\Gamma$ if $\Gamma'$ is irreducible and equivalent to $\Gamma$.
\end{definition}
Note that each universal constraint $(v_i \star v_j)$ in $\Gamma$ is, by definition, always a redundant constraint in $\Gamma$. We call this a \emph{trivial} redundant constraint. {In the following, we give an example of non-trivial redundant RCC5 constraints.}

\begin{figure}[htb]
\centering
\begin{tabular}{c}
\includegraphics[width=.25\columnwidth]{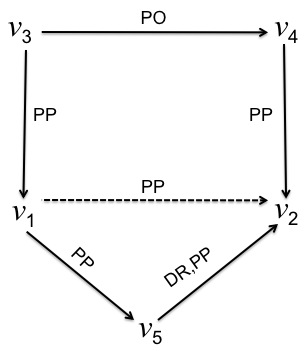} \\
\end{tabular}
\caption{An RCC5 network, where $(\pp{v_1}{v_2})$ is redundant.}\label{fig:pp-fig}
\end{figure}

\begin{example}
Suppose
\begin{align*}
\Gamma &=\{  \pp{v_1}{v_2},\pp{v_1}{v_5},\pp{v_3}{v_1}, 
\pp{v_4}{v_2}, v_5\{\bdr,\bpp\} v_2,v_3\bpo v_4\}.
\end{align*}
Then $(\pp{v_1}{v_2})$ is redundant. This is because, after enforcing path-consistency to $\Gamma\setminus\{(\pp{v_1}{v_2})\}$, we have $(\pp{v_5}{v_2})$ {and hence} $(\pp{v_1}{v_2})$. This shows that $\Gamma\setminus\{(\pp{v_1}{v_2})\}$ entails  $(\pp{v_1}{v_2})$. Moreover, $(\pp{v_1}{v_2})$ is the only non-trivial redundant constraint in $\Gamma$ and $\Gamma\setminus\{(\pp{v_1}{v_2})\}$ is the {unique} prime subnetwork of $\Gamma$.
\end{example}

Given {a qualitative constraint network} $\Gamma$, a very interesting question is, \emph{how to find a prime subnetwork of $\Gamma$}? This problem is clearly at least as hard as determining if $\Gamma$ is reducible. Similar to the case of propositional logic formulae \cite{Liberatore05}, we have the following result for RCC5/8.

\begin{proposition}\label{prop:redundant-co-npc} 
Let $\Gamma$ be an RCC5/8 network and suppose $(x R y)$ is a constraint in $\Gamma$. It is co-NP-complete to decide if $(x  R y)$ is redundant in $\Gamma$. 
\end{proposition}
\begin{proof}
First of all, we note that $(x R y)$ is redundant in $\Gamma$ iff $(\Gamma\setminus\{(x R y)\}) \cup\{x R^c y\}$ is inconsistent, where $R^c$ is the complement of $R$. Since it is NP-complete to decide if an RCC5/8 network is consistent, we know this redundancy problem (i.e., the problem of determining if a constraint is redundant in a network) is in co-NP. On the other hand, it is easy to construct a polynomial many-one reduction from the inconsistency problem of RCC5/8 to the redundancy problem. {Fix two variables $x,y$. Suppose  $\Gamma$ is  an arbitrary RCC5/8 network over $V$ and $x,y$ are two variables in $V$. Then $\Gamma$ is inconsistent iff $\Gamma\setminus \{(x R y)\} \models (x R^c y)$ iff $( x R^c y)$ is redundant in $(\Gamma\setminus\{(x R y)\}) \cup\{x R^c y\}$.} This shows that the redundancy problem is co-NP complete.
\end{proof}
{Similarly, we can show that the redundancy problems for IA, CRA, and RA are also co-NP-complete and, because the consistency problem of PA is in P,  the redundancy problem for PA is  in P.}

To determine if a network $\Gamma$ is reducible, we need in the worst case check for $O(n^2)$ constraints in $\Gamma$ whether they are redundant in $\Gamma$. By the above proposition, this is a decision problem in  $\Delta_2^P${, the class of problems solvable in polynomial time with an oracle for some NP-complete problem.}  Finding a prime subnetwork of $\Gamma$ is more complicated. A naive method is to remove redundant constraints iteratively from $\Gamma$ until we get an irreducible network. Let $c_1,c_2,\ldots,c_k$ be the sequence of {all non-trivial constraints} in $\Gamma$. Write $\Gamma_0=\Gamma$, and define
\begin{align}\label{eq:Gamma_i}
\Gamma_{i+1}=
\left\{
\begin{array}{ll}
\Gamma_i\setminus\{c_{i+1}\} &\quad \mbox{if $c_{i+1}$ is redundant in $\Gamma_i$;} \\
\Gamma_i &\quad \mbox{if otherwise.}
\end{array}
\right.
\end{align}
for $0\leq i\leq k-1$. Then it is easy to show that $\Gamma_k$ is a prime subnetwork of $\Gamma$. Suppose we have an oracle which can tell if a constraint is redundant in a network. Then $\Gamma_k$ can be constructed in $O(n^2)$ time. {We note that the construction of the prime subnetwork $\Gamma_k$ depends on the  order  of the constraints $c_1,c_2, \ldots,c_k$.} 

Despite that it is in general co-NP-complete to determine if a constraint is redundant, we have a polynomial time procedure if the constraints are all taken from a tractable subclass of RCC5/8.

\begin{proposition}\label{prop:redundant-horn-5}
Let $\mathcal{S}$ be a tractable subclass of RCC5/8 that contains all basic relations. Suppose $\Gamma$ is a network over $\mathcal{S}$. Then in $O(n^3)$ time we can determine whether a constraint is redundant in $\Gamma$ and in $O(n^5)$ time  find all redundant constraints of $\Gamma$. In addition, a prime subnetwork for $\Gamma$ can be found in $O(n^5)$ time.
\end{proposition}
\begin{proof}
Suppose $(x R y)$ is a constraint in $\Gamma$ and let $\Gamma'\equiv \Gamma\setminus\{(x R y)\}$. To determine if $(x R y)$ is redundant in $\Gamma$, we check for each basic RCC5/8 relation $\gamma$ that is not in $R$, whether $\Gamma'\cup\{(x\gamma y)\}$ is consistent. If the answer is confirmative for one $\gamma$ (note that RCC5 has five and RCC8 has eight basic relations), then $(x R y)$ is not redundant in $\Gamma$. By Proposition~\ref{prop:horn5}, the consistency of $\Gamma'\cup\{(x\gamma y)\}$ can be determined by enforcing path-consistency and hence can be determined in cubic time. Since there are $O(n^2)$ constraints in $\Gamma$, in $O(n^5)$ time we can find all redundant constraints of $\Gamma$.  

Suppose $c_1,c_2,\ldots,c_k$  are all non-trivial constraints of $\Gamma$. We define $\Gamma_0=\Gamma$, and set $\Gamma_{i+1}$ as in \eqref{eq:Gamma_i}. Note that if a constraint is {non-redundant} in $\Phi$ then it is also {non-redundant} in any subset of $\Phi$. It is straightforward to show that $\Gamma_k$ is a prime subnetwork of $\Gamma$. Since we can determine in cubic time whether a constraint is redundant in a network over $\mathcal{S}$, $\Gamma_k$ can be computed in $k\times O(n^3)$ time, which is bounded by $O(n^5)$. 
\end{proof}
{Similar conclusions apply to other calculi. For example, since the consistency problem of PA can be solved in $O(n^2)$ time, the redundancy problem of PA can be solved in $O(n^2)$ time and we can find a prime subnetwork for any consistent PA network in $O(n^4)$ time.}

It is often interesting to know when a constraint is contained in \emph{some} or \emph{all} prime subnetworks. The following notion will be helpful in partially answering questions like this.
\begin{definition}\label{dfn:core}
{Let $\qcm$ be a qualitative calculus and suppose $\Gamma$ is a qualitative constraint network over $\qcm$.} Write $\Gamma_c$ for the set of non-redundant constraints in $\Gamma$. We call $\Gamma_c$ the \emph{core} of $\Gamma$. 
\end{definition}
It is easy to see that the core of $\Gamma$ is contained in every prime subnetwork of $\Gamma$. \emph{Are prime subnetworks unique? And, is the core itself always a prime subnetwork?}  The following example shows that in general this is not the case.
\begin{example}
\label{ex:2prime-networks}
Suppose $\Gamma$ is the RCC5 network specified as below
\begin{align*}
\{\p{v_1}{v_2}, \p{v_2}{v_3},\p{v_3}{v_1},\po{v_1}{v_4},\po{v_2}{v_4}\},
\end{align*}
where $\bp=\{\bpp,\beq\}$. Then both $\bpo$ constraints in $\Gamma$ are redundant. This is because, enforcing path-consistency to $ \{\p{v_1}{v_2}, \p{v_2}{v_3},\p{v_3}{v_1}\}$ we have $\eq{v_1}{v_2}, \eq{v_1}{v_3},\eq{v_2}{v_3}$. Therefore, knowing one $\bpo$  constraint will infer the other. Moreover, $\Gamma$ has no other redundant constraints and $\{\p{v_1}{v_2}, \p{v_2}{v_3},\p{v_3}{v_2}\}$ is the core of $\Gamma$ {but not equivalent to $\Gamma$}. It is easy to see that $\Gamma_c\cup\{\po{v_1}{v_4}\}$ and $\Gamma_c\cup\{\po{v_2}{v_4}\}$ are two prime subnetworks of $\Gamma$.
\end{example}
Note that this occurs since there is a cycle of $\bp$ constraints in $\Gamma$, i.e., $\Gamma$ is $\bp$-cyclic. {In the following we  often assume that $\Gamma$ has the following property:}
\begin{align}\label{eq:noneq}
(\forall i,j)[ (i\not=j) \rightarrow (\Gamma \not\models (\eq{v_i}{v_j}))]. 
\end{align}
{This implies that no two variables are forced to be identical. We call a network which satisfies \eqref{eq:noneq} an \emph{all-different} constraint network. Note that an all-different network is always consistent, as an inconsistent network entails everything.}

{
The following proposition shows that the all-different requirement is not restrictive at all for constraint networks over a tractable subalgebra. 
\begin{proposition}\label{prop:alldiff}
Let $\mathcal{S}$ be a tractable subclass of RCC5/8 that contains all basic relations. Suppose $\Gamma=\{v_i R_{ij} v_j: 1\leq i,j\leq n\}$ is a consistent network over $\mathcal{S}$ and $\Gamma_p$ its a-closure. Then, for any $i\not=j$, $\Gamma \models (\eq{v_i}{v_j})$ iff $(\eq{v_i}{v_j})$ is in $\Gamma_p$.
\end{proposition}
\begin{proof}
The sufficiency part is clear. We only need {to} show the necessity part. Suppose $\Gamma\models (\eq{v_i}{v_j})$. We show $(\eq{v_i}{v_j})$ is in $\Gamma_p$. Suppose $\Gamma_p=\{v_i S_{ij} v_j: 1\leq i,j\leq n\}$. Because $\Gamma$ is consistent,  $\Gamma_p$ is path-consistent and each $S_{ij}$ is nonempty. By Theorem~21 of \cite{Renz99},  $\Gamma_p$ has a consistent scenario $\Gamma^*=\{v_i \alpha_{ij} v_j: 1\leq i,j\leq n\}$, where $\alpha_{ij}=\beq$ iff $S_{ij}=\beq$. In other words, if $S_{ij}\not=\beq$, i.e., $(\eq{v_i}{v_j})$ is not in $\Gamma_p$, then $\alpha_{ij}$ cannot be $\beq$ and hence $\Gamma$ does not entail $(\eq{v_i}{v_j})$. This is a contradiction and hence $(\eq{v_i}{v_j})$ is in $\Gamma_p$. 
\end{proof}
The above proposition shows that whether a constraint network is all-different can be answered by enforcing path-consistency. When a constraint network is not all-different, we can amalgamate those identical variables and thus obtain an equivalent but simplified all-different network.}

\vspace*{2mm}
In the next section we will show that, if $\Gamma$ is {an all-different} constraint network over a distributive subalgebra of RCC5/8, then $\Gamma_c$ is the unique prime subnetwork of $\Gamma$. This is quite surprising, as, in general, knowing that  $(x R y)$ and $(u S v)$ are both redundant in $\Gamma$ does not imply that $(u S v)$ is also redundant in $\Gamma\setminus\{(x R y)\}$. 

\section{Networks over a Distributive Subalgebra}

In this section, we assume $\mathcal{S}$ is a distributive subalgebra of RCC5/8. Let  $\Gamma$ be {an all-different network} over $\mathcal{S}$. {Because $\Gamma$ satisfies \eqref{eq:noneq}, there is in particular} no $\beq$ constraint in $\Gamma$. We show that $\Gamma_c$, the core of $\Gamma$, is equivalent to $\Gamma$ and hence the unique prime subnetwork of $\Gamma$. Using this result, we then further give a cubic time algorithm for computing the unique prime subnetwork of $\Gamma$.


To prove that $\Gamma_c$ is equivalent to $\Gamma$, we need two important results.  The first result, stated in Theorem~\ref{thm:minimal}, shows that the a-closure of $\Gamma$ is minimal, i.e. $\Gamma_p$ is exactly $\Gamma_m$. The second result, stated in Proposition~\ref{prop:redun+}, shows that a particular constraint $(x R y)$ is redundant in $\Gamma$ iff its corresponding constraint in $\Gamma_p$ is redundant in $\Gamma_p$. Our main result, stated in Theorem~\ref{thm:core=Gamma+}, then follows directly from these two results. 

{In Section 4.1, we prove Theorem~\ref{thm:minimal}; in Section 4.2, we characterise relations in such a minimal network in terms of the weak compositions of paths from $x$ to $y$ in $\Gamma$; and in Section 4.3 we prove Proposition~\ref{prop:redun+}. Using these results, we show in Section 4.3 that $\Gamma_c$ is equivalent to $\Gamma$ and hence the unique prime subnetwork of $\Gamma$ and give in Section 4.4 a cubic time algorithm for computing $\Gamma_c$.}

\subsection{The A-closure of $\Gamma$ Is Minimal}

{To prove that a network is minimal, by Proposition~\ref{wgcISmin}, we only need {to} show that it is \emph{weakly globally consistent} in the sense of Definition~\ref{dfn:minimal-network}.}

\begin{theorem}\label{thm:wglobal}
Let $\mathcal{S}$ be a distributive subalgebra of RCC5/8. Suppose $\Gamma=\{v_i R_{ij} v_j: 1\leq i,j\leq n\}$ is a path-consistent network over $\mathcal{S}$. Then $\Gamma$ is weakly globally consistent.
\end{theorem}
\begin{proof} Write $V_{k}=\{v_1,v_2,\ldots,v_k\}$ for $1\leq k < n$. Without loss of generality, we only show that every consistent scenario of $\Gamma{\downarrow}_{V_k}$ can be extended to a consistent scenario of $\Gamma{\downarrow}_{V_{k+1}}$. Suppose $\Delta=\{v_i \delta_{ij} v_j: 1\leq i,j\leq k\}$ is a consistent scenario of $\Gamma{\downarrow}_{V_k}$. Then each $\delta_{ij}$ is a basic relation in $R_{ij}$. For each $1\leq i\leq k$, write $T_i$ for $R_{k+1,i}$ {(see Figure~\ref{fig:thm:wglobalProof} for illustration)}. Let $\widehat{T}_i=\bigcap_{j=1}^k T_j\circw \delta_{ji}$. 

 \begin{figure}[h]
 \centering
  \includegraphics[width=0.5\textwidth]{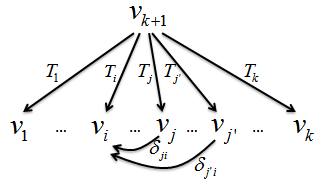}
  \caption{{Illustration of $\Gamma{\downarrow}_{V_{k+1}}$ in the proof of Theorem~\ref{thm:wglobal}}}\label{fig:thm:wglobalProof}
\end{figure}

{We assert that each $\widehat{T}_i$ is nonempty. By Lemma~\ref{lem:3x}, it is easy to show by induction on $k$ that, for any set $\{W_1,W_2,\ldots,W_k\}$ of $k\geq 3$ nonempty relations  in $\mathcal{S}$, $\bigcap_{i=1}^k W_i \not=\varnothing$ iff $W_i\cap W_j\not=\varnothing$ for any $1\leq i\not=j\leq n$. Therefore, to show $\widehat{T}_i\not=\varnothing$, we only need {to} show $T_j\circw \delta_{ji} \cap T_{j'}\circw \delta_{j'i} \not=\varnothing$ for any $1\leq i \leq n$ and any $1\leq j\not=j'\leq n$.  Applying the cycle law as stated in Proposition~\ref{prop:ra} twice, we have 
\begin{align*}
T_j\circw \delta_{ji} \cap (T_{j'}\circw \delta_{j'i}) \not=\varnothing &\quad \mbox{iff}\quad  (T_{j'}\circw \delta_{j'i})\circw (\delta_{ji})^{-1} \cap T_j\not=\varnothing\\
&\quad \mbox{iff}\quad T_{j'}\circw (\delta_{j'i}\circw \delta_{ij}) \cap T_j\not=\varnothing\\
&\quad \mbox{iff}\quad T_{j'}^{-1}\circw T_j \cap (\delta_{j'i}\circw \delta_{ij})\not=\varnothing\\
&\quad \mbox{iff}\quad (R_{j',k+1}\circw R_{k+1,j}) \cap (\delta_{j'i}\circw \delta_{ij})\not=\varnothing.
\end{align*}
Note here $(\delta_{ji})^{-1}=\delta_{ij}$ and $T_{j'}^{-1}=(R_{k+1,j'})^{-1}=R_{j',k+1}$. Because $\delta_{j'j}\subseteq R_{j'j}$, we know $\delta_{j'j}\cap R_{j'j}\not=\varnothing$. Since both $\Delta$ and $\Gamma$ are path-consistent, we also have $\delta_{j'j}\subseteq \delta_{j'i}\circw \delta_{ij}$ and $R_{j'j} \subseteq R_{j',k+1}\circw R_{k+1,j}$. Therefore, we have $(R_{j',k+1}\circw R_{k+1,j}) \cap (\delta_{j'i}\circw \delta_{ij}) \supseteq R_{j'j}\cap \delta_{j'j}\not=\varnothing$ and  hence $T_j\circw \delta_{ji} \cap T_{j'}\circw \delta_{j'i} \not=\varnothing$. This shows that $\widehat{T}_i\not=\varnothing$ for any $1\leq i\leq n$.
}

To show that $\Gamma{\downarrow}_{V_k} \cup\{v_{k+1} \widehat{T}_i v_i : 1\leq i\leq k\}$ is path-consistent, we only need {to} show for $1\leq i\not=i'\leq k$ that $\widehat{T}_i\circw \delta_{ii'}\supseteq \widehat{T}_{i'}$. By the distributivity and $\delta_{ji}\circw\delta_{ii'}\supseteq \delta_{ji'}$ we have  
\begin{align*}
\widehat{T}_i\circw \delta_{ii'} & = (\bigcap_{j=1}^k T_j\circw \delta_{ji}) \circw \delta_{ii'}  = \bigcap_{j=1}^k T_j \circw (\delta_{ji} \circw \delta_{ii'})\supseteq \bigcap_{j=1}^k T_j \circw \delta_{ji'}  = \widehat{T}_{i'}.
\end{align*}
This shows that $\Gamma{\downarrow}_{V_k} \cup\{v_{k+1} \widehat{T}_i v_i : 1\leq i\leq k\}$ is path-consistent and hence{, by Corollary~\ref{cor:distr-pc},} has a consistent scenario $\Delta'$. It is clear that $\Delta'$ extends $\Delta$ from $V_{k}$ to $V_{k+1}$. Because  $\Gamma{\downarrow}_{V_k} \cup\{v_{k+1} \widehat{T}_i v_i : 1\leq i\leq k\}$ refines $\Gamma{\downarrow}_{V_{k+1}}$, we know  $\Gamma{\downarrow}_{V_{k+1}}$ has a consistent scenario which extends $\Delta$. 
\end{proof}

{Together with Proposition~\ref{wgcISmin}, the above result immediately implies that the a-closure of a consistent network $\Gamma$ over a distributive subalgebra is minimal.}
\begin{theorem}\label{thm:minimal}
Let $\mathcal{S}$ be a distributive subalgebra of RCC5/8. Suppose $\Gamma$ is a consistent network over $\mathcal{S}$ and $\Gamma_p$ its a-closure. Then $\Gamma_p$ {is identical to} the minimal network of $\Gamma$.
\end{theorem}
{The above results can also be extended to distributive subalgebras of PA, IA and CRA, but do not hold in general. Consider the network $\Gamma$ over $\horn_5$ shown in Figure~\ref{fig:18}{, which is inspired by a network over PA in \cite{Beek89}}. The network is path-consistent but not minimal. In fact, the relation $\beq$ in the constraint $(v_1\{\bpp,\beq\} v_2)$ is not feasible,  i.e., there exists no solution of $\Gamma$ in which $(v_1\beq v_2)$ is satisfied. By Proposition~\ref{wgcISmin}, we know $\Gamma$ is not weakly globally consistent.
}

\begin{figure}[htb]
\centering
\includegraphics[width=0.35\columnwidth]{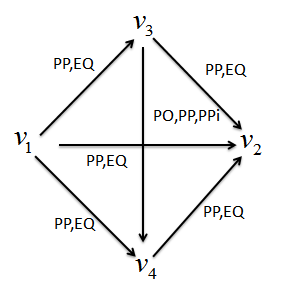} 
{
\caption{A counter-example of Theorem~\ref{thm:minimal}: a path-consistent constraint network $\Gamma$ over $\horn_5$.}
\label{fig:18}
}
\end{figure}

{In the next subsection, we characterise relations in such a minimal network in terms of the weak compositions of paths from $x$ to $y$ in $\Gamma$.}

{\subsection{Weak Compositions of Paths}}
{Let $\qcm$ be a qualitative calculus. A qualitative constraint network $\Gamma$ is in essence a labelled directed graph consisting of the variables in $\Gamma$ as vertices and qualitative relations in $\qcm$ between the variables as labels.} A \emph{path} $\pi$ from a variable $x$ to another variable $y$ is a sequence of constraints $c_1, c_2, ..., c_s$ such that $c_i=(u_{i-1} R_i u_i)$ and $u_0=x,u_s=y$.  The weak composition of path $\pi$ is the qualitative relation in $\qcm$ defined as 
$$\ct(\pi)\equiv R_1 \circw R_2  \circw ... \circw R_s.$$ 
{Since the weak composition operation is associative, the relation $\ct(\pi)$ defined above is unambiguous.} We say a path $\pi$ is \emph{non-trivial} if $\ct(\pi)$ is not the universal relation. Note that $(x \ct(\pi) y)$ {is} entailed by those constraints in $\pi$. 

Suppose $\Gamma$ is a constraint network over a distributive subalgebra of RCC5/8, {$(x R y)$ and $(x S y)$ are respectively the constraints in $\Gamma$ and $\Gamma_p$ that relate $x$ to $y$. We next show that $S$ is the intersection of the weak compositions of all paths from $x$ to $y$ in $\Gamma$. Note that such a path may contain $(x R y)$ as an (or the unique) edge.}

\vspace*{2mm}
\begin{lemma}\label{lem:int-paths}
Let $\mathcal{S}$ be a distributive subalgebra of RCC5/8. Suppose $\Gamma$ is a consistent network over $\mathcal{S}$ and $\Gamma_p$ its a-closure. Assume furthermore that $(x S y)$ is a constraint in $\Gamma_p$. Then $S$ is the intersection of the weak compositions of all paths from $x$ to $y$ in $\Gamma$.
\end{lemma}
\begin{proof}
Suppose the network becomes stable in $k$ steps when enforcing PCA. For $l\leq k$, we write $R_{ij}^l$ for the constraint between $v_i$ and $v_j$ in the $l$-th step. We prove by using induction on $l$ that every $R_{ij}^l$ is the intersection of the weak compositions of \emph{several} paths from $v_i$ to $v_j$ in $\Gamma$. 

When $l=0$, this is clearly true. Suppose this is true for $l\leq s$. We show it also holds for $l=s+1$. Suppose in this step the following updating rule is called
\begin{align*}
R_{ij}^{l+1}= (R_{ik}^l \circw R_{kj}^l) \cap R_{ij}^l.
\end{align*}
By induction hypothesis, we know $R_{ij}^l$ is the intersection of the weak compositions of several paths from $v_i$ to $v_j$ in $\Gamma$. Similar for $R_{ik}^l$ and $R_{kj}^l$. Note that when joining a path from $v_i$ to $v_k$ and a path from $v_k$ to $v_j$, we obtain a path from $v_i$ to $v_j$. Because every constraint in $\Gamma$ is taken from $\mathcal{S}$, in which weak composition distributes over nonempty intersections, it follows that  $R_{ik}^l \circw R_{kj}^l$ is identical to the intersection of the weak compositions of all these paths from $v_i$ to $v_j$ via $v_k$. It is now clear that $R_{ij}^{l+1}$ also satisfies the property. 

So far, we have shown for every constraint $(x S y)$ in $\Gamma_p$ that $S$ is the intersection of the weak compositions of \emph{several} paths from $x$ to $y$ in $\Gamma$. Because $\Gamma_p$ is path-consistent, the weak composition of \emph{every} path from $x$ to $y$ in $\Gamma_p$ contains $S$. Therefore, $S$ is also contained in the intersection of the weak compositions of \emph{all} paths from $x$ to $y$ in $\Gamma$. This shows that $S$ is exactly the intersection of the weak compositions of \emph{all} paths from $x$ to $y$ in $\Gamma$. 
\end{proof}
{The distributive property  is necessary in the above lemma. Consider the consistent RCC5 network $\Gamma$ over $\horn_5$ shown in Figure~\ref{fig:16}. The intersection of the weak compositions of all paths from $v_1$ to $v_2$ in $\Gamma$ is $\{\bdr,\bpp\}$, while the relation that relates $v_1$ to $v_2$ in $\Gamma_p$ is $\{\bdr\}$, which is strictly contained in $\{\bdr,\bpp\}$.}

\begin{figure}[htb]
\centering
\begin{tabular}{cc}
\includegraphics[width=0.35\columnwidth]{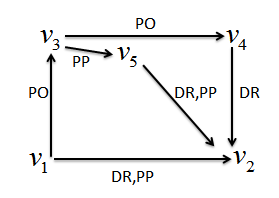} 
&
\includegraphics[width=0.35 \columnwidth]{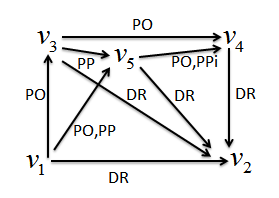}\\
(a) 
&
(b)
\end{tabular}
{
\caption{A counter-example of Lemma~\ref{lem:int-paths}: (a) a constraint network $\Gamma$ over $\horn_5$; and (b) its a-closure $\Gamma_p$.}
\label{fig:16}}
\end{figure}

The following lemma shows that the weak composition of a cycle contains $\beq$ and $\bpo$. {This result holds for arbitrary RCC5/8 networks which are all-different.}
\begin{lemma}\label{lem:circle}
{Suppose $\Gamma$ is an all-different RCC5/8 network and $\pi=(c_1, c_2, ..., c_s)$ ($s\geq 2$)} a path from $x$ to itself in $\Gamma$ such that $c_i=(u_{i-1} R_i u_i)$, $u_0=u_s=x$. Then $\ct(\pi)$ contains $\bo_5\equiv\{\bpo,\bpp,\bppi,\beq\}$ if $\Gamma$ is an RCC5 network, and contains $\bo_8\equiv\{\bpo,\btpp,\btppi,\beq\}$ if $\Gamma$ is an RCC8 network. 
\end{lemma} 
\begin{proof}
Write $y$ for $u_1$. Let $R=R_1$ and $T=\ct(\pi_{>1})=R_2\circw R_3 \circw \ldots \circw R_s$. Note that $y\not=x$ and $\pi_{>1}$ is a path from $y$ to $x$. Suppose $S$ is the relation from $x$ to $y$ in the a-closure of $\Gamma$. Because $\Gamma$ is consistent,  we know $S$ is nonempty and $S\subseteq R$, $S\subseteq T^{-1}$. Furthermore, since $\Gamma$ {is all-different and hence} satisfies \eqref{eq:noneq}, we know $S\not=\{\beq\}$.  As a consequence, we know there is a basic RCC5/8 relation $\alpha\not=\beq$ which is contained in $R\cap T^{-1}$. Therefore, $\ct(\pi)=R\circw T\supseteq\alpha \circw \alpha^{-1}$. By checking the composition tables of RCC5 and RCC8, we have that $\alpha \circw \alpha^{-1}$ (hence $\ct(\pi)$) contains $\bo_5$ ($\bo_8$, respectively) for any RCC5 (RCC8, respectively) basic relation $\alpha\not=\beq$.
\end{proof}
The following lemma provides a \emph{finer} characterisation of the constraint $(x S y)$ in $\Gamma_p$ in terms of paths in $\Gamma$ {that do not contain the constraint $(x R y)$}.
 \begin{lemma}\label{prop:+circle+}
Let $\mathcal{S}$ be a distributive subalgebra of RCC5/8. Suppose $\Gamma$ is {an all-different network over $\mathcal{S}$} and $\Gamma_p$ its a-closure. Assume that $(x R y)$ and $(x S y)$ are the constraints from $x$ to $y$ in $\Gamma$ and $\Gamma_p$ respectively. Then $S=R\cap W$, where $W$ is the intersection of the weak compositions of all paths from $x$ to $y$ in $\Gamma\setminus\{(x R y)\}$.
\end{lemma} 
\begin{proof}
{Because $(x R y)$ is the only path with length 1  from $x$ to $y$ in $\Gamma$, Lemma~\ref{lem:int-paths} in fact asserts that}  $S$ is the intersection of $R$ and the weak compositions of all paths in $\Gamma$ with length $\geq 2$. Note that each path from $x$ to $y$ in $\Gamma\setminus\{(x R y)\}$ has length $\geq 2$. We know $S\subseteq R\cap W$. 

To show $S\supseteq R\cap W$, we only need {to} show $\ct(\pi)\supseteq R\cap W$ for every path from $x$ to $y$ in $\Gamma$ with length $\geq 2$. Suppose $\pi=(c_1, c_2, ..., c_s)$ ($s\geq 2$)  is such a path and $c_i=(u_{i-1} R_i u_i)$, $u_0=x,u_s=y$. 

\begin{figure}[h]
 \centering
\begin{tabular}{ccc}
\includegraphics[width=0.25\textwidth]{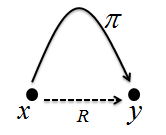}
 &\includegraphics[width=0.4\textwidth]{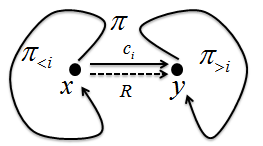}
 & \includegraphics[width=0.25\textwidth]{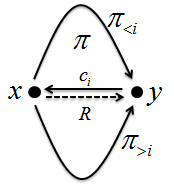}
\end{tabular}
 \caption{{Illustration of the three types of paths: (a) Case 1, (b) Case 2, (c) Case 3, where solid lines represent constraints or paths contained in $\pi$ and the dashed line represents the  constraint $(xRy)$ in $\Gamma$.}}
 \label{fig:prop:+circle+Proof}
\end{figure}

There are three types of paths {(see Figure~\ref{fig:prop:+circle+Proof} for illustration)}.

Case 1. $\pi$ contains neither $(x R y)$ nor $(y R^{-1} x)$. Clearly $\pi$ is a path from $x$ to $y$ in $\Gamma\setminus\{(x R y)\}$. By definition we have $\ct(\pi) \supseteq W$.

Case 2. If  $c_i=(x R y)$ for some $1\leq i\leq s$, then $\ct(\pi)=\ct(\pi_{<i})\circw R \circw \ct(\pi_{>i})$. Note that either $\pi_{<i}$ or $\pi_{>i}$ is a nonempty cycle. By Lemma~\ref{lem:circle} we know the weak composition of each cycle contains $\beq$. Therefore, we know $\ct(\pi)\supseteq R$.

Case 3. If $c_i=(y R^{-1} x)$ for some $1\leq i\leq s$, then $\ct(\pi)=\ct(\pi_{<i})\circw\ct(\pi_{\geq i})$. Without loss of generality, we assume $c_i$ is the first constraint in $\pi$ such that $c_i=(y R^{-1} x)$. It is clear that $\pi_{<i}$ is a path of Case 1 and hence $W\subseteq \ct(\pi_{<i})$. Note that $\pi_{\geq i}$ is a path from $y$ to itself. By Lemma~\ref{lem:circle} we know $\beq\in \ct(\pi_{\geq i})$ hence $\ct(\pi)=\ct(\pi_{<i})\circw\ct(\pi_{\geq i})\supseteq W\circw \beq =W$.

This shows that $R\cap W$ is contained in the weak composition of every path from $x$ to $y$ in $\Gamma$ with length $\geq 2$. Since $S$ is the intersection of $R$ and all paths from $x$ to $y$ in $\Gamma$ with length $\geq 2$, this shows that $S\supseteq R\cap W$. Therefore we have $S=R\cap W$.
\end{proof}
{As Lemma~\ref{lem:int-paths}, the above result does not hold in general. Consider the network shown in Figure~\ref{fig:16} and the constraint from $v_1$ to $v_2$. We have $R=\{\bdr,\bpp\}$, $S=\{\bdr\}$, but $R\cap W=\{\bdr,\bpp\}\not=S$.
}

\subsection{Correspondence Between Redundant Constraints in $\Gamma$ and $\Gamma_p$}

Suppose $\Gamma$ is an RCC5/8 network over a distributive subalgebra $\mathcal{S}$ and $\Gamma_p$ its a-closure. Let $(x R y)$ and $(x S y)$ be the constraints from $x$ to $y$ in $\Gamma$ and $\Gamma_p$ respectively. {We prove that $(x R y)$ is redundant in $\Gamma$ iff $(x S y)$ is redundant in $\Gamma_p$. To this end, we need several lemmas.} 

{The following two lemmas show that a constraint $(x R y)$ in $\Gamma$ is redundant iff $R$ contains the intersection of the weak compositions of all paths from $x$ to $y$ in $\Gamma\setminus\{(x R y)\}$.
\begin{lemma}\label{cor:intesection-redundant-if}
Suppose $\Gamma$ is a consistent RCC5/8 network and $(x R y)$ a constraint in $\Gamma$. Assume that $W$ is the intersection of the weak compositions of all paths from $x$ to $y$ in $\Gamma\setminus\{(x R y)\}$. Then $(x R y)$ is redundant in $\Gamma$ if $R\supseteq W$.
\end{lemma}
\begin{proof}
Write $\Gamma'\equiv \Gamma\setminus\{(x R y)\}$. For every path $\pi$ from $x$ to $y$ in $\Gamma'$, we know $\Gamma'$ entails $(x \ct(\pi) y)$. By the definition of $W$, this implies that $\Gamma'$ entails $(x W y)$. Suppose $R\supseteq W$. It is clear that every solution of $\Gamma'$ also satisfies $(x R y)$, and therefore, $(x R y)$ is redundant in $\Gamma$.
\end{proof}
}
\begin{lemma}\label{cor:intesection-redundant-onlyif}
Let $\mathcal{S}$ be a distributive subalgebra of RCC5/8. Suppose $\Gamma$ is {an all-different} network over $\mathcal{S}$ and $(x R y)$ is a constraint in $\Gamma$. Assume that $W$ is the intersection of the weak compositions of all paths from $x$ to $y$ in $\Gamma\setminus\{(x R y)\}$. Then $(x R y)$ is redundant in $\Gamma$ only if $R\supseteq W$.
\end{lemma}
\begin{proof}
Suppose $(x R y)$ is redundant in $\Gamma$. Then each solution of $\Gamma'=\Gamma\setminus\{(x R y)$ also satisfies $(x R y)$. Write $(x T y)$ for the constraint between $x$ and $y$ in $\Gamma'_p$, the a-closure of $\Gamma'$. By {Lemma~\ref{prop:+circle+}} we know $T=W$. Furthermore, {by Theorem~\ref{thm:minimal}}, we know each basic relation in $T$ is feasible in $\Gamma'$. This implies that $T=W$ is contained in $R$.  
\end{proof}
{This result does not hold in general. Consider the constraint network $\Gamma$ over $\horn_5$ shown in Figure~\ref{fig:22} and the constraint from $v_1$ to $v_2$. It is easy to show that $\Gamma$ is path-consistent, i.e., $\Gamma=\Gamma_p$, and  $(v_1\{\bpp\} v_2)$ is redundant in $\Gamma$. Furthermore, we have $W=\{\bpp,\beq\}$, which is not contained in $R=\{\bpp\}$.
}
\begin{figure}[htb]
\centering
\includegraphics[width=0.35\columnwidth]{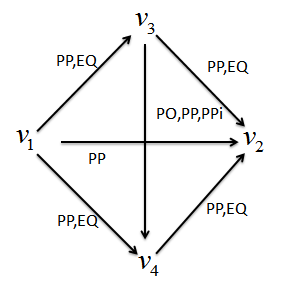} 
{
\caption{A counter-example of Lemma~\ref{cor:intesection-redundant-onlyif}: a path-consistent constraint network $\Gamma=\Gamma_p$ over $\horn_5$.}
\label{fig:22}
}
\end{figure} 

{The above characterisation of redundant constraints can be strengthened if $\Gamma$ is itself path-consistent.}

\begin{lemma}\label{lem:pcnetwork-redun}
Let $\mathcal{S}$ be a distributive subalgebra of RCC5/8. Suppose $\Gamma$ is an all-different and path-consistent network over $\mathcal{S}$. Then a constraint $(v_i R_{ij} v_j)$ is redundant in $\Gamma$ iff $R_{ij}=\bigcap\{ R_{ik}\circw R_{kj}: k\not=i,j\}$, i.e., $R_{ij}$ is the intersection of the weak compositions of all paths from $v_i$ to $v_j$ which have length $2$.
\end{lemma}
\begin{proof}
{
Let $W_{ij}$ be the intersection of the weak compositions of all paths from $v_i$ to $v_j$ in $\Gamma\setminus\{(v_i R_{ij} v_j)\}$. It is clear $W_{ij}\subseteq \bigcap\{ R_{ik}\circw R_{kj}: k\not=i,j\}$.
}
{
Suppose $R_{ij}=\bigcap\{ R_{ik}\circw R_{kj}: k\not=i,j\}$. We have $R_{ij} \supseteq W_{ij}$. By Lemma~\ref{cor:intesection-redundant-if}, this immediately implies that $(v_i R_{ij} v_j)$ is redundant in $\Gamma$.
}
{
On the other hand, suppose $(v_i R_{ij} v_j)$ is redundant in $\Gamma$. We show $R_{ij}=\bigcap\{ R_{ik}\circw R_{kj}: k\not=i,j\}$. By Lemma~\ref{cor:intesection-redundant-onlyif} we know $R_{ij}\supseteq W_{ij}$. Let $\pi=(c_1, c_2, ..., c_s)$ ($s\geq 2$)  be an arbitrary path from $v_i$ to $v_j$ in $\Gamma\setminus\{(v_i R_{ij} v_j)\}$ such that $c_k=(u_{k-1} R_k u_k)$, $u_0=v_i,u_s=v_j$.  Then $\ct(\pi)=R_1\circw \ct(\pi_{>1})$. Suppose $u_1=v_{i'}$.  Then $R_1=R_{ii'}$ and $\pi_{>1}$ is a path from $v_{i'}$ to $v_j$. Because $\Gamma$ is path-consistent, we know by Lemma~\ref{lem:int-paths} that  $R_{i'j}$ is contained in $\ct(\pi_{>1})$.  This implies that $\ct(\pi)$ contains $R_{ii'}\circw R_{i'j}$ and, therefore, $\bigcap\{ R_{ik}\circw R_{kj}: k\not=i,j\}$. Due to the arbitrariness of $\pi$, $W_{ij}$ also contains $\bigcap\{ R_{ik}\circw R_{kj}: k\not=i,j\}$. Since $R_{ij}\supseteq W_{ij}$, we have $R_{ij}\supseteq \bigcap\{ R_{ik}\circw R_{kj}: k\not=i,j\}$. By the path-consistency of $\Gamma$, we have $R_{ij}\subseteq R_{ik}\circw R_{kj}$ for every $k\not=i,j$. This shows $R_{ij} =\bigcap\{ R_{ik}\circw R_{kj}: k\not=i,j\}$.
}
\end{proof}

{This result does not hold in general. Again, consider the path-consistent RCC5 network $\Gamma$ over $\horn_5$ shown in Figure~\ref{fig:22}. Although $(\pp{v_1}{v_2})$ is redundant in $\Gamma$, $R_{13}\circw R_{32} \cap R_{14}\circw R_{42} = \{\bpp,\beq\}$ strictly contains $\{\bpp\}$.
}

\vspace*{2mm}

{We next show that $(x R y)$ is redundant in $\Gamma$ iff  $(x S y)$ is redundant in $\Gamma_p$.} 

 \begin{lemma}\label{prop:redun-onlyif}
Suppose $\Gamma$ is {an all-different RCC5/8 network}. Assume that $(x R y)$ and $(x S y)$ are the constraints from $x$ to $y$ in $\Gamma$ and $\Gamma_p$ respectively. Then $(x R y)$ is redundant in $\Gamma$ only if $(x S y)$ is redundant in $\Gamma_p$.
\end{lemma} 
\begin{proof}
Write $\Gamma'$ and ${\Gamma''}$ for $\Gamma\setminus\{(x R y)\}$ and $\Gamma_p\setminus\{(x S y)\}$ respectively. 

Suppose $(x R y)$ is redundant in $\Gamma$. Then $\Gamma'$ entails $(x R y)$. Note that ${\Gamma''}$ refines $\Gamma'$. We know every solution of ${\Gamma''}$ is a solution of $\Gamma'$, hence also satisfies $(x R y)$. In other words, each solution of ${\Gamma''}$ is a solution of $\Gamma$. Since $\Gamma$ is equivalent to its a-closure, we know each solution of ${\Gamma''}$ is also a solution of $\Gamma_p$, hence also satisfies $(x S y)$. Therefore, $(x S y)$ is redundant in $\Gamma_p$.
\end{proof}

 \begin{proposition}\label{prop:redun}
Let $\mathcal{S}$ be a distributive subalgebra of RCC5/8. Suppose $\Gamma$ is {an all-different network over $\mathcal{S}$}. Assume that $(x R y)$ and $(x S y)$ are the constraints from $x$ to $y$ in $\Gamma$ and $\Gamma_p$ respectively. Then $(x R y)$ is redundant in $\Gamma$ iff $(x S y)$ is redundant in $\Gamma_p$.
\end{proposition} 

\begin{proof}
{
The necessity part has been proved in Lemma~\ref{prop:redun-onlyif}. We only need {to} show the sufficiency part.}  
Write $\Gamma'$ and $\Gamma''$ for $\Gamma\setminus\{(x R y)\}$ and $\Gamma_p\setminus\{(x S y)\}$ respectively. Suppose $(x S y)$ is redundant in $\Gamma_p$. Let $W$ be the intersection of the weak compositions of all paths from $x$ to $y$ in $\Gamma\setminus\{(x R y)\}$.  {To show that $(x R y)$ is redundant in $\Gamma$, by Lemma~\ref{cor:intesection-redundant-if}, we only need {to} show $R\supseteq W$.} 

{Recall $S=R\cap W$ by Lemma~\ref{prop:+circle+}. To show $R\subseteq W$,  we first show} 
\begin{align}\label{eq:s=rw}
R\cap W \supseteq W\cap \bo_l\circw R \cap R\circw \bo_l,
\end{align}
where $\bo_l$ is either $\bo_5$ or $\bo_8$ {(cf. Lemma~\ref{lem:circle} for definition)}, according to whether $\Gamma$ is over RCC5 or RCC8.

{Because $(x S y)$ is redundant in $\Gamma_p$, by Lemma~\ref{lem:pcnetwork-redun}, we know $S$ is the intersection of the weak compositions of all paths with length 2 from $x$ to $y$ in $\Gamma_p$.} For each  constraint  $(u_i S_{ij} u_j)$ in any such a path, Lemma~\ref{lem:int-paths} shows that $S_{ij}$ is the intersection of the weak compositions of all paths from $u_i$ to $u_j$ in $\Gamma$. Replace each $(u_i S_{ij} u_j)$ with several paths such that $S_{ij}$ is the intersection of their weak compositions. We get several paths from $x$ to $y$ in $\Gamma$ with length $\geq 2$ such that $S$ is the intersection of the weak compositions of these paths. By Lemma~\ref{lem:int-paths} again we know $S$ is contained in the weak composition of every path from $x$ to $y$ in $\Gamma$. This shows that $S$ is exactly the intersection of the weak compositions of \emph{all} paths from $x$ to $y$ in $\Gamma$ with length $\geq 2$.\footnote{{While we can further show that $S$ is the intersection of the weak compositions of all paths from $x$ to $y$ in $\Gamma$ that have no cycles and are with length $\geq 2$, it is not guaranteed that such a path is in $\Gamma\setminus\{ (x R y)\}$. That is, we cannot directly show $S=W$.}}  

As we have seen in the proof of Lemma~\ref{prop:+circle+}, there are three types of paths. For every path $\pi$ of Case 1 or 3 {(defined  in Lemma~\ref{prop:+circle+})}, we know $\ct(\pi)$ contains $W$. Suppose $\pi$ is a path of Case 2 and $c_i=(x R y)$ for some $1\leq i \leq s$. Then $\ct(\pi)=\ct(\pi_{<i})\circw R\circw \ct(\pi_{>i})$. Note that if $\pi_{<i}$ ($\pi_{>i}$, respectively) is nonempty, then $\ct(\pi_{<i})$ ($\ct(\pi_{>i})$, respectively) contains $\bo_l$ by Lemma~\ref{lem:circle}.  Either $\pi_{<i}$ or $\pi_{>i}$ is a cycle. Therefore, $\ct(\pi)$ contains $\bo_l\circw R \cap R\circw \bo_l \cap \bo_l\circw R \circw \bo_l$. In summary, for each path $\pi$ from $x$ to $y$ in $\Gamma$ with length $\geq 2$, we have $\ct(\pi)\supseteq W\cap \bo_l\circw R \cap R\circw \bo_l \cap \bo_l\circw R \circw \bo_l$. Because $ \bo_l\circw R \circw \bo_l$ is always the universal relation (as $\bpo\circw R\circw \bpo =\bpo\circw\bpo=\star_l$ {by Lemma~\ref{lemma:po/dr/composition}}), we know {$S$, as the intersection of the weak compositions of \emph{all} paths from $x$ to $y$ in $\Gamma$ with length $\geq 2$, contains $W\cap \bo_l\circw R \cap R\circw \bo_l$. Since $S=R\cap W$, we have \eqref{eq:s=rw} immediately.} 

We next show $R\supseteq W$. {Because $\Gamma$ is consistent and satisfies \eqref{eq:noneq}, we know $S=R\cap W$ is neither empty nor $\{\beq\}$, i.e., }
\begin{align*}
\varnothing\not= R\cap W\not=\{\beq\}.
\end{align*}

If $\bpo\in R$, then $\bo_l\circw R \cap R\circw \bo_l\supseteq \bpo\circw\bpo$ is the universal relation. That $R\supseteq W$ follows directly from $R\cap W\supseteq W\cap \star_l= W$.

If $\bpo\not\in R$, then $\bpo\not\in W$ because $\bpo\in \bo_l\circw R \cap R\circw \bo_l$ and \eqref{eq:s=rw} holds. We show $R\supseteq W$. We only consider RCC8 relations. The case for RCC5 relations is similar. Suppose $R$ is a relation in a distributive subalgebra of RCC8 such that $\bpo\not\in R$ and $R\not=\beq$. {Checking the lists of relations in the two maximal distributive subalgebras given in~\ref{app:dis},} $R$ is either a basic relation other than \bpo\ and \beq, or one of the following relations 
\begin{align}\label{eq:rcc8-dist-1} 
 &\{\btpp,\bntpp\},\{\btppi,\bntppi\}, \\ 
&\{\bdc,\bec\}, \{\btpp,\beq\},\{\btppi,\beq\},\notag\\ 
&\{\btpp,\bntpp,\beq\},\{\btppi,\bntppi,\beq\}.\notag
\end{align} 

There are several subcases. Suppose $R$ is a basic relation $\alpha$ other than \bpo\ and \beq.  We write $\alpha^d$ for the other basic relation such that $\{\alpha,\alpha^d\}$ is a relation in \eqref{eq:rcc8-dist-1}. For example, $\bdc^d=\bec$, $\btpp^d=\bntpp$, and ${\btppi}^d=\bntppi$. From the RCC8 composition table we can see 
\begin{align*}
\{\alpha,\alpha^d,\bpo\} &\subseteq \bpo\circw \alpha \cap \alpha \circw \bpo \subseteq \bo_8\circw \alpha \cap \alpha \circw \bo_8 
\end{align*}
holds for every basic relation $\alpha$ other than $\bpo$\ and $\beq$. We assert that $\alpha^d\not\in W$ if $R=\{\alpha\}$. This is because, otherwise, we have $\alpha^d\in W\cap \bo_8\circw R \cap R \circw \bo_8$ and hence by \eqref{eq:s=rw} $\alpha^d\in R\cap W\subseteq R$. A contradiction. In particular, if $\alpha$ is $\bdc$, $\bec$, $\bntpp$, or $\bntppi$, then $W=R$. If $\alpha$ is either $\btpp$\ or $\btppi$, then we can further show that $\beq\in \bo_8\circw \alpha \cap \alpha\circw \bo_8$ and hence $\beq\not\in W$. This implies that $W=R$. 

Suppose $R$ is $\{\bdc,\bec\}$, $\{\btpp,\bntpp,\beq\}$, or $\{\btppi,\bntppi,\beq\}$. Note that $\bpo\not\in W$, and $\varnothing \not= R\cap W \not=\{\beq\}$. This shows that $W$ is contained in $R$. 

Suppose $R$ is $\{\btpp,\bntpp\}$ or $\{\btppi,\bntppi\}$. By \eqref{eq:s=rw} and $\beq\in   \bo_8\circw R \cap R \circw \bo_8$  we know that $W$ does not contain $\beq$. Hence $W$ is contained in $R$.

Suppose $R$ is $\{\btpp,\beq\}$. By \eqref{eq:s=rw} and  $\bntpp\in  \bo_8\circw R \cap R \circw \bo_8$, $W$ cannot contain $\bntpp$. This implies that $W$ is contained in $R$. The case for $R=\{\btppi,\beq\}$ is similar.

In summary, we have ${R\supseteq W}$ in all cases. In other words, $R$ can be obtained as the intersection of all paths from $x$ to $y$ in $\Gamma\setminus\{(x R y)\}$. Hence $(x R y)$ is redundant in $\Gamma$ by Lemma~\ref{cor:intesection-redundant-if}.
\end{proof}

{The  result does not hold in general. Consider the constraint network $\Gamma$ over $\horn_5$ shown in Figure~\ref{23CE} and the constraint from $v_3$ to $v_2$.  It is clear that the constraint $(v_3 \bpp v_2)$ is redundant in $\Gamma_p$. However, $(v_3 \bpp v_2)$ is not redundant in $\Gamma$. This is because $(\dr{v_3}{v_2})$ is consistent with $\Gamma\setminus\{(v_3 \bpp v_2)\}$ (shown in Figure~\ref{23CE}(c)). Actually, it is easy to construct a solution $\{a_1,a_2,a_3,a_4\}$ of $\Gamma\setminus\{(v_3 \bpp v_2)\}$    in which $(\pp{a_3}{a_1}),\ (\pp{a_1}{a_4})$ and $(\dr{a_2}{a_j})$ for $j=1,3,4$. }
\begin{figure}[h]
\centering
\begin{tabular}{cccc}
\hspace*{-4mm}\includegraphics[width=0.33\textwidth]{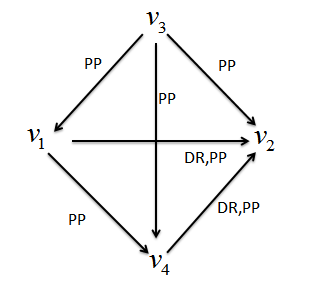}
& {\includegraphics[width=0.31\textwidth]{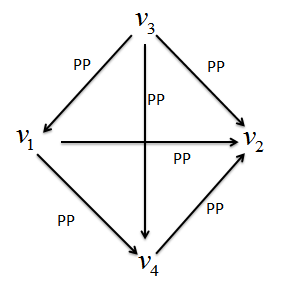}}
&{\includegraphics[width=0.34\textwidth]{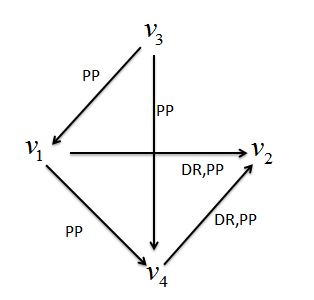}}\\
 (a) & (b) & (c) 
 \end{tabular}
 {
 \caption{A counter-example of Proposition~\ref{prop:redun}: (a) a constraint network $\Gamma$ over $\mathcal{H}_5$; (b) its a-closure $\Gamma_p$; (c)  $\Gamma\setminus\{(v_3 \bpp v_2)\}$. }\label{23CE}
 }
\end{figure}

Recall that Theorem~\ref{thm:minimal} asserts that $\Gamma_p$ is minimal. Proposition~\ref{prop:redun} can be rephrased as follows:

 \begin{proposition}\label{prop:redun+}
Let $\mathcal{S}$ be a distributive subalgebra of RCC5/8. Suppose $\Gamma$ is {an all-different network over $\mathcal{S}$} and $\Gamma_m$ the minimal network of $\Gamma$. Assume that $(x R y)$ and $(x S y)$ are the constraints from $x$ to $y$ in $\Gamma$ and $\Gamma_m$ respectively. Then $(x R y)$ is redundant in $\Gamma$ iff $(x S y)$ is redundant in $\Gamma_m$.
\end{proposition}

As a consequence, we have our main result.

\begin{theorem}
\label{thm:core=Gamma+}
Let $\mathcal{S}$ be a distributive subalgebra of RCC5/8. Suppose $\Gamma$ is {an all-different network over $\mathcal{S}$} and $\Gamma_c$ the core of $\Gamma$. Then $\Gamma_c$ is equivalent to $\Gamma$ and hence the unique prime subnetwork of $\Gamma$.
\end{theorem}
\begin{proof}
Suppose $c_1,c_2,\ldots,c_k$ are the redundant constraints of $\Gamma$. Let $\Gamma_0=\Gamma$ and $\Gamma_{i+1}=\Gamma_i\setminus\{c_{i+1}\}$ for $0\leq i\leq k$. Note that $\Gamma_k$ is precisely $\Gamma_c$, the core of $\Gamma$. Suppose $0\leq i< k$ is the largest integer such that $\Gamma_i$ is equivalent to $\Gamma$.  

Suppose $c_{i+1}=(x R y)$ and $(x S y)$ is the corresponding constraint in $\Gamma_m$, the minimal network of $\Gamma$. Note that $c_{i+1}$ is also in $\Gamma_i$. By Proposition~\ref{prop:redun+} we know $(x S y)$ is redundant in $\Gamma_m$ since $(x R y)$ is redundant in $\Gamma$. Because $\Gamma_m$ is also the minimal network of $\Gamma_i$, by Proposition~\ref{prop:redun+} again we know $(x R y)$ is redundant in $\Gamma_i$. This means that $\Gamma_{i+1}$ is equivalent to $\Gamma_i$, hence $\Gamma$. This contradicts our assumption that $i<k$ is the largest integer such that $\Gamma_i$ is equivalent to $\Gamma$. Therefore, $i=k$ and $\Gamma_c$ is equivalent to $\Gamma$. 
\end{proof}

{The above result does not hold in general. For example, consider the RCC5 network $\Gamma$ over $\horn_5$ shown in Figure~\ref{25CE}(a). The core $\Gamma_c$ (shown in Figure~\ref{25CE}(b)) is not equivalent to and hence not a prime subnetwork of $\Gamma$. This is because  $(\dr{v_3}{v_2})$ is feasible in $\Gamma_c$ but not in $\Gamma$.} 
\begin{figure}[h]
\centering
 \begin{tabular}{cc}
 {\includegraphics[width=0.35\textwidth]{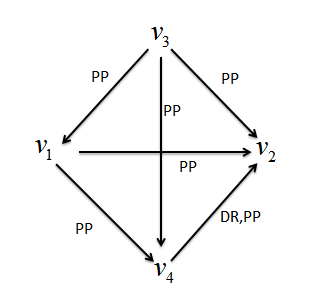}}
 & {\includegraphics[width=0.32\textwidth]{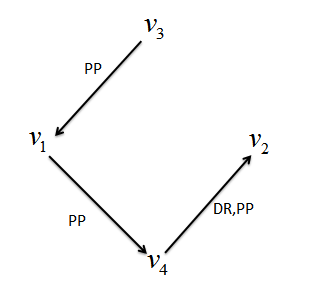}}\\
 (a) & (b)
 \end{tabular}
 {
 \caption{A counter-example of Theorem~\ref{thm:core=Gamma+}:  (a) a constraint network $\Gamma$ over $\mathcal{H}_5$; (b) the core $\Gamma_c$ of $\Gamma$. 
 }\label{25CE}
 }
\end{figure}

Similar to Proposition~\ref{prop:redundant-horn-5}, we can show that the core of an RCC5/8 network over a tractable subclass can be found in $O(n^5)$ time. In the next subsection we show this can be improved if the network is over a distributive subalgebra.
\subsection{A Cubic Time Algorithm for Computing the Core of $\Gamma$}

Suppose $\Gamma$ is {an all-different}  network over a distributive subalgebra of RCC5/8. Proposition~\ref{prop:redun} and Lemma~\ref{lem:pcnetwork-redun} suggest a simple way for computing $\Gamma_c$, the unique prime subnetwork of $\Gamma$. By Proposition~\ref{prop:redun}, a constraint $(v_i R_{ij} v_j)$ in $\Gamma$ is redundant iff the corresponding constraint $(v_i S_{ij} v_j)$ in $\Gamma_p$ is redundant. Furthermore, Lemma~\ref{lem:pcnetwork-redun} shows that $(v_i S_{ij} v_j)$ is redundant in $\Gamma_p$ iff $S_{ij}$ is the intersection of all $S_{ik}\circw S_{kj}$ ($k\not=i,j$). We hereby have the following cubic time algorithm {(Algorithm~\ref{al:ax})} for finding all redundant constraints in $\Gamma$. For each constraint $(v_i S_{ij} v_j)$, to verify if $S_{ij}=\bigcap\{S_{ik}\circw S_{kj}: k\not=i,j\}${, we introduce a temporary relation $Q_{ij}$, which is initially assigned as the universal relation. It is possible that, after just a few intersections of $S_{ik}\circw S_{kj}$ with $Q_{ij}$, the resulting $Q_{ij}$ is already equal to $S_{ij}$, which implies $(v_i S_{ij} v_j)$ is redundant in $\Gamma_p$ and hence $(v_i R_{ij} v_j)$ is redundant in $\Gamma$. }

\begin{algorithm} \label{al:ax}
\caption{Algorithm for finding {all} redundant constraints in {a constraint network} over a distributive subalgebra $\mathcal{S}$ of RCC5/8, where $\star_l$ is the universal relation in RCC$l$.} \BlankLine
\LinesNumbered
{
\KwIn{A consistent {RCC5/8 network} $\Gamma = \{v_i R_{ij} v_j: 1\leq i,j\leq n\}$ over $\mathcal{S}$ { and $V=\{v_i:1\leq i\leq n\}$.}}
\KwOut{$Redun$: the set of redundant constraints of $\Gamma$.}
\BlankLine
$Redun \leftarrow \varnothing$\;
$\Gamma_p\ {\leftarrow}\  \mbox{ the a-closure of }\ \Gamma$\;
\For{each constraint $(v_i S_{ij} v_j) \in \Gamma_p$}
{
$Q_{ij} \leftarrow \star_l$\;
\For{each variable $v_k \in V\setminus\{v_i,v_j\}$}
{
$Q_{ij} \leftarrow Q_{ij}\cap S_{ik}\circw S_{kj}$\;
\If{$Q_{ij} = S_{ij}$}
{
$Redun \leftarrow Redun\cup \{(v_i R_{ij} v_j)\}$\;
break the inner loop;
}
}
}
}
\end{algorithm}

\vspace*{-3mm}
\section{Related Works}
Redundancy checking is an important task in AI research, in particular in knowledge representation and reasoning. For example, Ginsberg \cite{Ginsberg88a} and Schmolze and Snyder \cite{Schmolze99} designed algorithms for checking redundancy of knowledge bases; Gottlob and Ferm{\"u}ller \cite{Gottlob93} and Liberatore  \cite{Liberatore05} analysed the computational properties of removing redundancy from a clause and a CNF formula, respectively; and Grimm and Wissmann \cite{GrimmW11} considered checking redundancy of ontologies. 

In research on constraint satisfaction problems (CSPs), there are also many studies of constraint redundancy. While most of this research concerns redundant modelling (e.g., \cite{ChoiLS07}), Chmeiss et al. \cite{ChmeissKS08} studied redundancy modulo a given local consistency. Their paper is close in spirit to ours. Let $\Gamma$ be a CSP and $\phi$ a local consistency. Chmeiss et al. call a constraint $c$ in $\Gamma$  \emph{$\phi$-redundant} iff $(\Gamma\setminus\{c\} )\cup \{\neg c\}$ is $\phi$-inconsistent. Because path-consistency implies consistency for RCC5/8 constraint networks over their maximal tractable subclasses, our notion of redundancy (when restricted to networks over these tractable subclasses) is equivalent to redundancy modulo path-consistency in the sense of \cite{ChmeissKS08}.

In qualitative spatial reasoning, there are also two works that are close to our research. Egenhofer and Sharma noticed in \cite{EgenhoferS93}  that ``For any scene description, the set of $n^2$ binary topological relations between the $n$ objects is redundant since some of these topological relations are always implied by others.'' They did not provide an efficient algorithm for deriving such a minimal set even for basic topological constraints. Recently, Wallgr{\"u}n \cite{Wallgrun12} proposed two algorithms to approximately find the prime subnetwork. The essence of these algorithms is to replace $R_{ij}$ with the universal constraint if there exists $k$ such that $R_{ik}\circw R_{kj}$ is contained in $R_{ij}$. {As was noted in \cite{Wallgrun12}, neither of these two algorithms is guaranteed to provide the optimal simplification. But it is worth noting that these two approximate algorithms are applicable to general constraint networks which are not necessarily over a distributive subalgebra. In Section~6, we will empirically compare our method with the methods of  Wallgr{\"u}n.}



The property of distributivity was first used by van Beek \cite{Beek89} for IA, but the notion of distributive subalgebra is new. It is not difficult to show that {PA}, IA,  RCC5 and RCC8 all have two maximal distributive subalgebras (see \ref{app:dis} for maximal distributive subalgebras of RCC5/8). Very interestingly, the two maximal distributive subalgebras of IA are exactly the subalgebras $\mathcal{C}_{IA}$ and $\mathcal{S}_{IA}$ discussed in \cite{AmaneddineC12}, where Amaneddine and Condotta proved that $\mathcal{C}_{IA}$ and $\mathcal{S}_{IA}$ are the only maximal subalgebras of IA {over which} path-consistent networks are globally consistent. For RCC8, the maximal distributive subalgebra $\mathcal{D}_{41}^8$ we identify in \ref{app:dis} turns out to be the class of convex RCC8 relations found in \cite{ChandraP05}, where Chandra and Pujari proved that path-consistent networks over $\mathcal{D}_{41}^8$ are minimal. In \ref{app:dis} we find another maximal distributive subalgebra for RCC8, which contains 64 relations. Furthermore, we also show that every path-consistent constraint network $\Gamma$ over a distributive subalgebra is {weakly globally consistent and minimal. This has not been studied for RCC5/8 before.}

\section{{Empirical Evaluation}}
\label{sec:empirical_evaluation}

{
In this section, we empirically evaluate our method in comparison with the methods of Wallgr{\"u}n \cite{Wallgrun12}. In \cite{Wallgrun12}, Wallgr{\"u}n  proposes two greedy algorithms for removing redundant constraints in the constraint network: the \emph{basic} and \emph{extended} simplification algorithms (hereafter Simple and SimpleExt). The Simple algorithm loops through all triples of regions $i$, $j$, and $k$ and identifies as redundant any constraints $R_{ik}$ such that $R_{ij}\diamond R_{jk}\subseteq R_{ik}$. A drawback of the Simple algorithm is that redundant relations removed may affect subsequent iterations of the algorithm.  Hence, the order in which triples are visited by the Simple algorithm can alter the resulting subnetwork. The SimpleExt solves this issue by first marking potentially redundant relations for removal, subject to a consistency check, before removing all marked relations in a final loop. The Simple and SimpleExt algorithms are not guaranteed to provide an optimal solution. Thus, the prime subnetwork is necessarily a (possibly improper) {subnetwork} of that generated by the SimpleExt algorithm, which is in turn a (possibly improper) {subnetwork} of that generated by the Simple algorithm.
}

\begin{figure}[htb]
\centering{
\begin{tabular}{cc}
\includegraphics[width=0.4\textwidth]{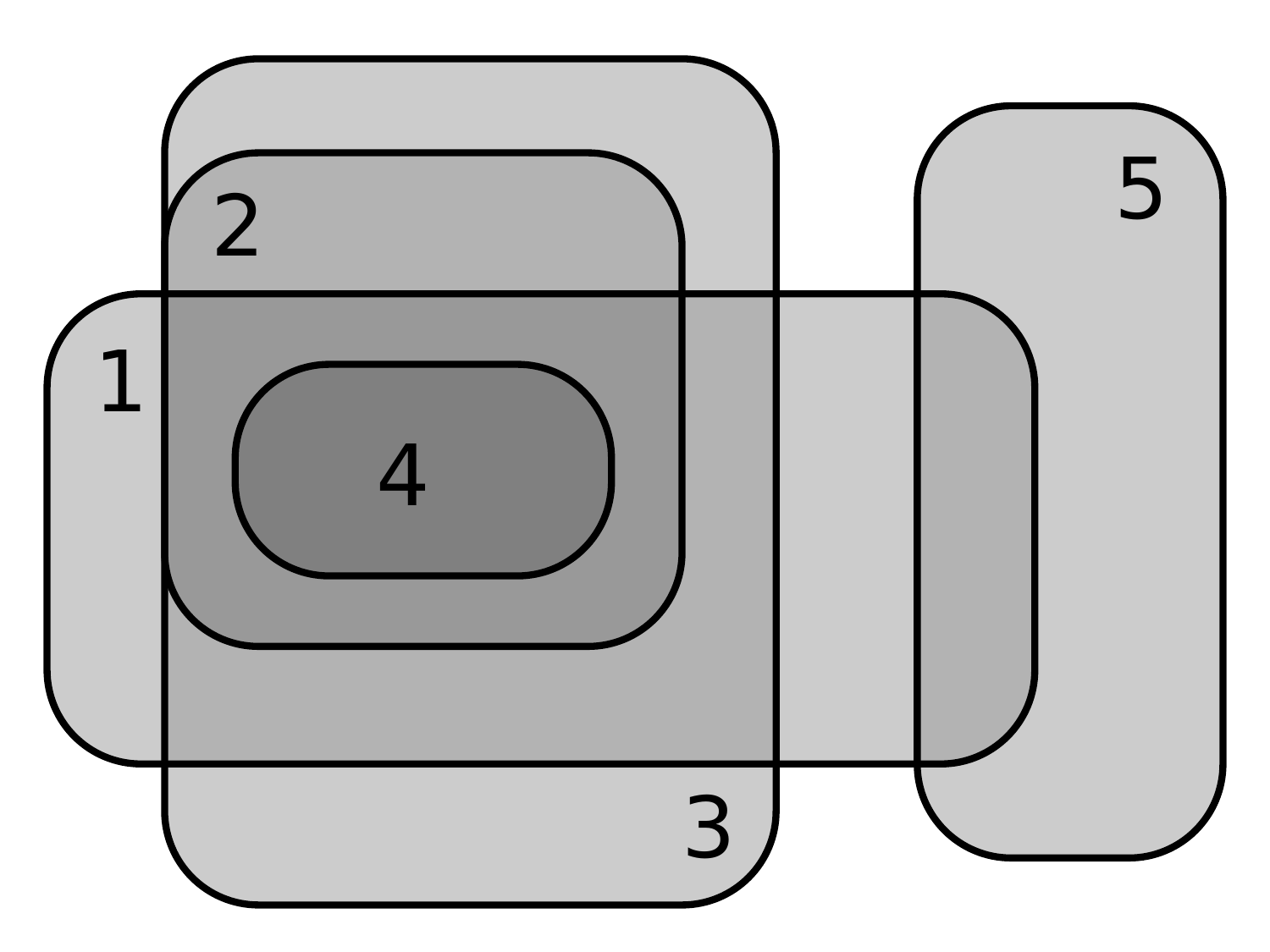}&
\includegraphics[width=0.4\textwidth]{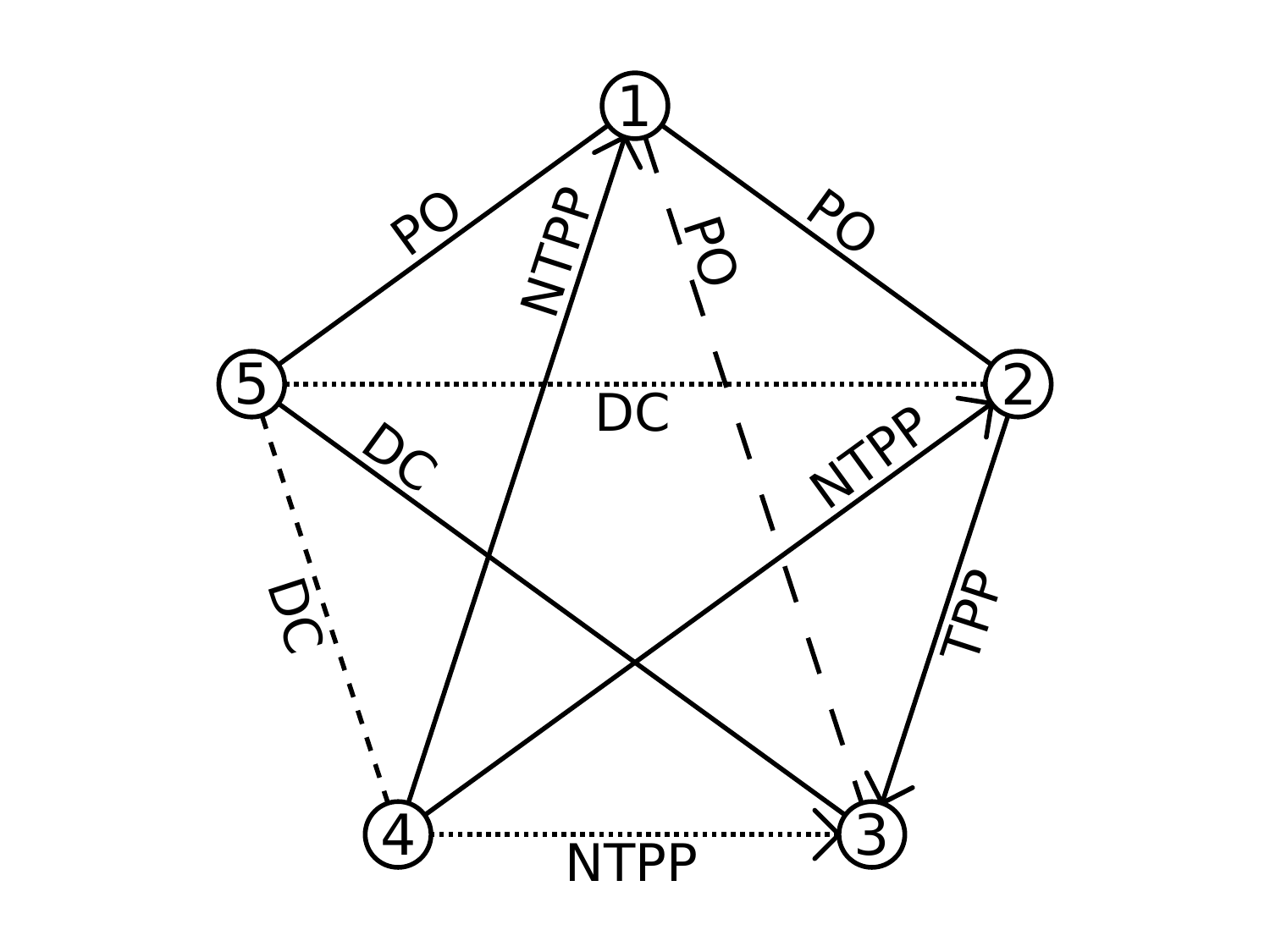}\\
\footnotesize{a.~Example region configuration}. & \footnotesize{b.~Example constraint network}\\
\end{tabular}}
\caption{{Example constraint network illustrating the differences between the prime subnetwork, and the subnetworks generated by the Simple and SimpleExt algorithms \cite{Wallgrun12}. Redundant constraints {found} in the prime subnetwork only are shown with wide dashes; constraints {found} in both the prime and SimpleExt subnetworks are shown with narrow dashes; and constraints {found} in the prime, SimpleExt, and Simple subnetworks are shown with dotted lines.}}\label{fig:example}
\end{figure}

{
Figure \ref{fig:example} shows an example with just five regions, highlighting the constraints identified as redundant in the prime subnetwork and by the Simple and SimpleExt algorithms. Assuming the Simple algorithm visits regions in numerical order, the relations between regions 2 and 5 and between 3 and 4 will be identified as redundant. Additionally, the SimpleExt algorithm is able to identify the relation between 4 and 5 as redundant. However, only in the prime subnetwork is the redundant relation between 1 and 3 identified. 
}

\subsection{{Data Sets}}

{In the following evaluation, two real data sets were used: the {UK} geographic ``footprint" dataset introduced in Section \ref{sec:motivation} (total {3443} regions) and the statistical areas levels 1--4 dataset for Tasmania (in total {1559} regions), provided by the Australian Bureau of Statistics. Derived from social media, the footprint data set contains a variety of regions of differing sizes and shapes, but relatively unstructured sharing almost no adjacent boundaries. In stark contrast the Tasmanian statistical areas are highly structured, made up of four levels of spatially contiguous and nested but non-overlapping regions. To aid in our analysis, {five} subsets of each of the two datasets were generated  {in addition to the full datasets}. The subsets were generated from selecting those regions that intersect an arbitrarily selected spatial region of increasing size. In this way, subsets of data of varying sizes were generated, with {108, 217, 433, 862 and 1725} regions in subsets of the footprint data set, and {49, 98, 193, 374 and 658} regions from the statistical areas set.}

{Subsets of spatially related regions were explicitly used, as opposed to entirely random selection of regions, to ensure that the range of {RCC8} base relations in each subset were representative of the {RCC8} relations in close spatial proximity. The resulting distribution of {RCC8} base relations in the full constraint network for these {10} region subsets {along with their complete dataset} is shown in Table \ref{tab:constraint_proportion}. By design, the relations exhibit systematic variations in the distribution of relations, for example with the statistical areas data set exhibiting consistently higher levels of \bdc\ and lower levels of \bpo\ relations (due to the non-overlapping nature of statistical areas); and smaller subsets exhibiting lower levels of \bdc\ relations (as a result of the smaller spatial area in which regions must fit for the smaller subsets of data). }


\begin{table}[htb]
\centering{\begin{tabular}{c|r|r|r|r|r|r|}
\cline{2-7}
\multicolumn{1}{l|}{}                                     & \multicolumn{1}{c|}{Size} & \multicolumn{1}{c|}{DC}        & \multicolumn{1}{c|}{EC}        & \multicolumn{1}{c|}{PO}        & \multicolumn{1}{c|}{NTPP(I)}   & \multicolumn{1}{c|}{TPP(I)}            \\ \hline
\multicolumn{1}{|c|}{}                                    & 5778                      & \cellcolor[HTML]{C0C0C0}1.1\%  & 0.0\%                          & \cellcolor[HTML]{9B9B9B}85.6\% & \cellcolor[HTML]{C0C0C0}13.3\% & \cellcolor[HTML]{C0C0C0}\textless0.1\% \\
\multicolumn{1}{|c|}{}                                    & 23436                     & \cellcolor[HTML]{9B9B9B}66.9\% & 0.0\%                          & \cellcolor[HTML]{9B9B9B}22.8\% & \cellcolor[HTML]{C0C0C0}10.3\% & \cellcolor[HTML]{C0C0C0}\textless0.1\% \\
\multicolumn{1}{|c|}{}                                    & 93528                     & \cellcolor[HTML]{9B9B9B}26.1\% & 0.0\%                          & \cellcolor[HTML]{9B9B9B}56.7\% & \cellcolor[HTML]{9B9B9B}17.1\% & \cellcolor[HTML]{C0C0C0}\textless0.1\% \\
\multicolumn{1}{|c|}{}                                    & 371091                    & \cellcolor[HTML]{9B9B9B}62.6\% & 0.0\%                          & \cellcolor[HTML]{9B9B9B}30.5\% & \cellcolor[HTML]{C0C0C0}6.9\%  & \cellcolor[HTML]{C0C0C0}\textless0.1\% \\
\multicolumn{1}{|c|}{}                                    & 1486950                   & \cellcolor[HTML]{9B9B9B}78.1\% & 0.0\%                          & \cellcolor[HTML]{9B9B9B}15.2\% & \cellcolor[HTML]{C0C0C0}6.7\%  & \cellcolor[HTML]{C0C0C0}\textless0.1\% \\
\multicolumn{1}{|c|}{\multirow{-6}{*}{\rotatebox[origin=c]{90}{Footprint}}}         & 5925403                   & \cellcolor[HTML]{9B9B9B}92.5\% & 0.0\%                          & \cellcolor[HTML]{C0C0C0}4.8\%  & \cellcolor[HTML]{C0C0C0}2.7\%  & \cellcolor[HTML]{C0C0C0}\textless0.1\% \\ \hline
\multicolumn{1}{|c|}{}                                    & 1176                      & \cellcolor[HTML]{9B9B9B}69.6\% & \cellcolor[HTML]{9B9B9B}20.0\% & 0.0\%                          & \cellcolor[HTML]{C0C0C0}2.0\%  & \cellcolor[HTML]{C0C0C0}8.4\%          \\
\multicolumn{1}{|c|}{}                                    & 4753                      & \cellcolor[HTML]{9B9B9B}87.5\% & \cellcolor[HTML]{C0C0C0}7.0\%  & 0.0\%                          & \cellcolor[HTML]{C0C0C0}1.9\%  & \cellcolor[HTML]{C0C0C0}3.6\%          \\
\multicolumn{1}{|c|}{}                                    & 18528                     & \cellcolor[HTML]{9B9B9B}92.9\% & \cellcolor[HTML]{C0C0C0}4.3\%  & 0.0\%                          & \cellcolor[HTML]{C0C0C0}0.4\%  & \cellcolor[HTML]{C0C0C0}2.4\%          \\
\multicolumn{1}{|c|}{}                                    & 69751                     & \cellcolor[HTML]{9B9B9B}96.7\% & \cellcolor[HTML]{C0C0C0}1.8\%  & 0.0\%                          & \cellcolor[HTML]{C0C0C0}0.7\%  & \cellcolor[HTML]{C0C0C0}0.8\%          \\
\multicolumn{1}{|c|}{}                                    & 216153                    & \cellcolor[HTML]{9B9B9B}98.0\% & \cellcolor[HTML]{C0C0C0}1.1\%  & 0.0\%                          & \cellcolor[HTML]{C0C0C0}0.4\%  & \cellcolor[HTML]{C0C0C0}0.4\%          \\
\multicolumn{1}{|c|}{\multirow{-6}{*}{\rotatebox[origin=c]{90}{Statistical areas}}} & 1214461                   & \cellcolor[HTML]{9B9B9B}99.2\% & \cellcolor[HTML]{C0C0C0}0.5\%  & 0.0\%                          & \cellcolor[HTML]{C0C0C0}0.2\%  & \cellcolor[HTML]{C0C0C0}0.2\%          \\ \hline
\end{tabular}}
\caption{{Table showing the proportion of {RCC8} constraints for the two data sets and across the six region subsets.}}\label{tab:constraint_proportion}
\end{table}

\subsection{{Redundant Constraints}}

{In \cite{Wallgrun12}, the two conjectures are made that: a.~the Simple and SimpleExt are good approximations for removing all redundant relations; and b.~ that the Simple algorithm is in practice almost as good as the SimpleExt algorithm at removing redundant relations. In this section, we compare the three types of subnetwork (prime, SimpleExt, and Simple) in practice and in the context of these conjectures. Figure \ref{fig:RegionsVsRedundant} shows the growth in size of the three types of subnetwork across the six subsets of each of the two data sets. Several features are worth noting in {Figure~}\ref{fig:RegionsVsRedundant}:} 

{
\begin{itemize} 
\item All three subnetworks grow in size approximately linearly with the number of regions (coefficient of determination {$R^2>0.97$} in all cases, indicating a high level of fit between the data and the linear regression). Linear $O(n)$ growth is a lower bound on the space complexity of these subnetworks, since they must remain connected (and so must have at least $n-1$ edges). Thus, this result indicates all three algorithms are approaching optimal scalability in terms of space complexity. The only exception occurs with the Simple subnetwork and in the case of the statistical areas data set, which grows in size quadratically with the number of regions. 
{
\item The prime subnetwork is consistently smaller than the subnetwork generated by the Simple algorithm at all network sizes and is significantly smaller for larger networks. 
\item The SimpleExt subnetwork is significantly larger than the prime subnetwork and of similar size to the Simple subnetwork in the case of the footprint data set, while it is of similar size to the prime subnetwork and significantly smaller than the Simple subnetwork in the case of the statistical areas data set.
}
\end{itemize}
}

{
In summary, neither the Simple nor SimpleExt algorithm can be relied upon to identify as many redundant constraints as the prime subnetwork, although the SimpleExt algorithm may in some cases identify many more redundant constraints than the Simple algorithm (such as the statistical areas data set). 
}

\begin{figure}[htb]
\centering{
\begin{tabular}{cc}
\includegraphics[width=0.47\textwidth]{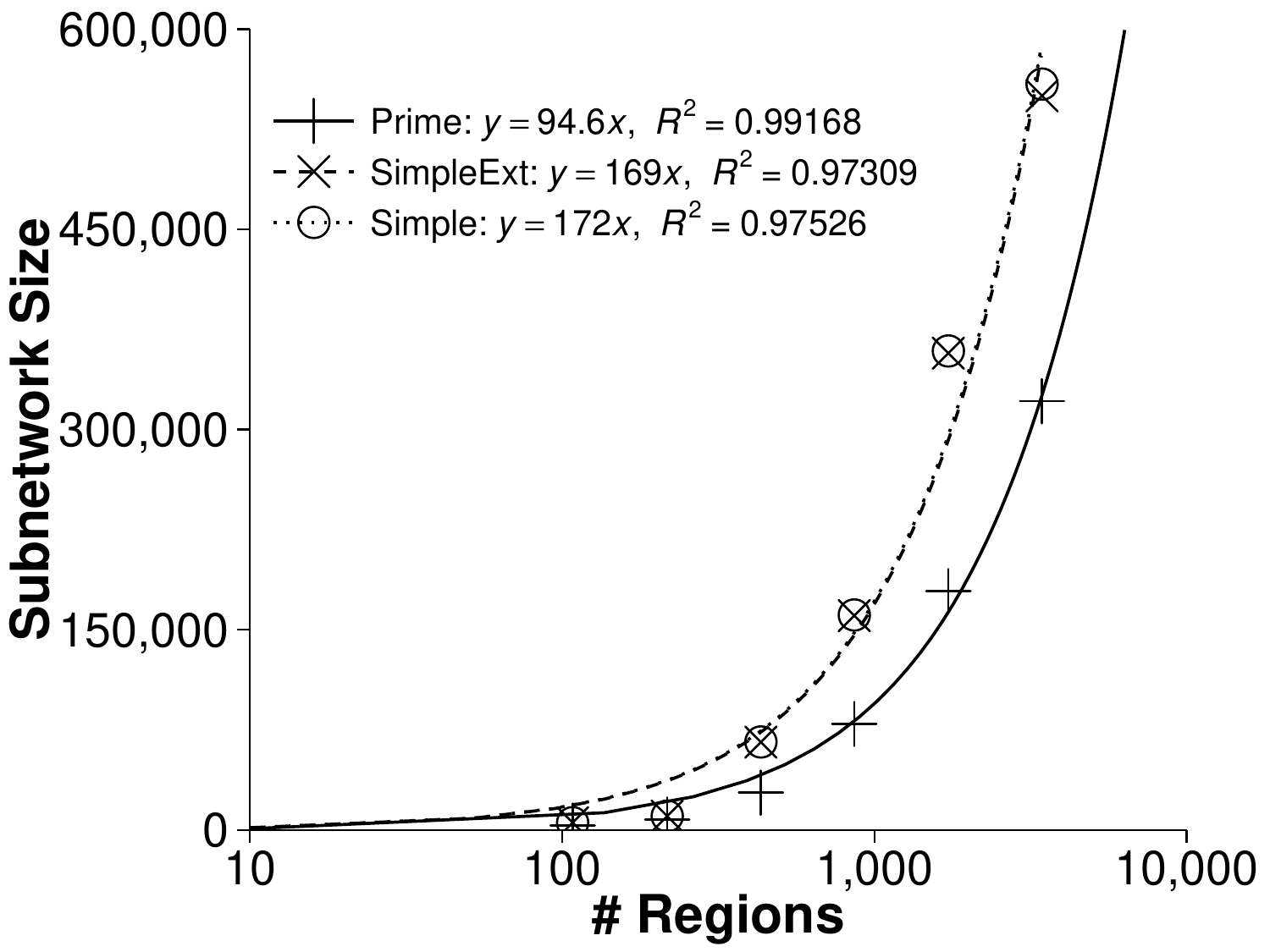}&
\includegraphics[width=0.47\textwidth]{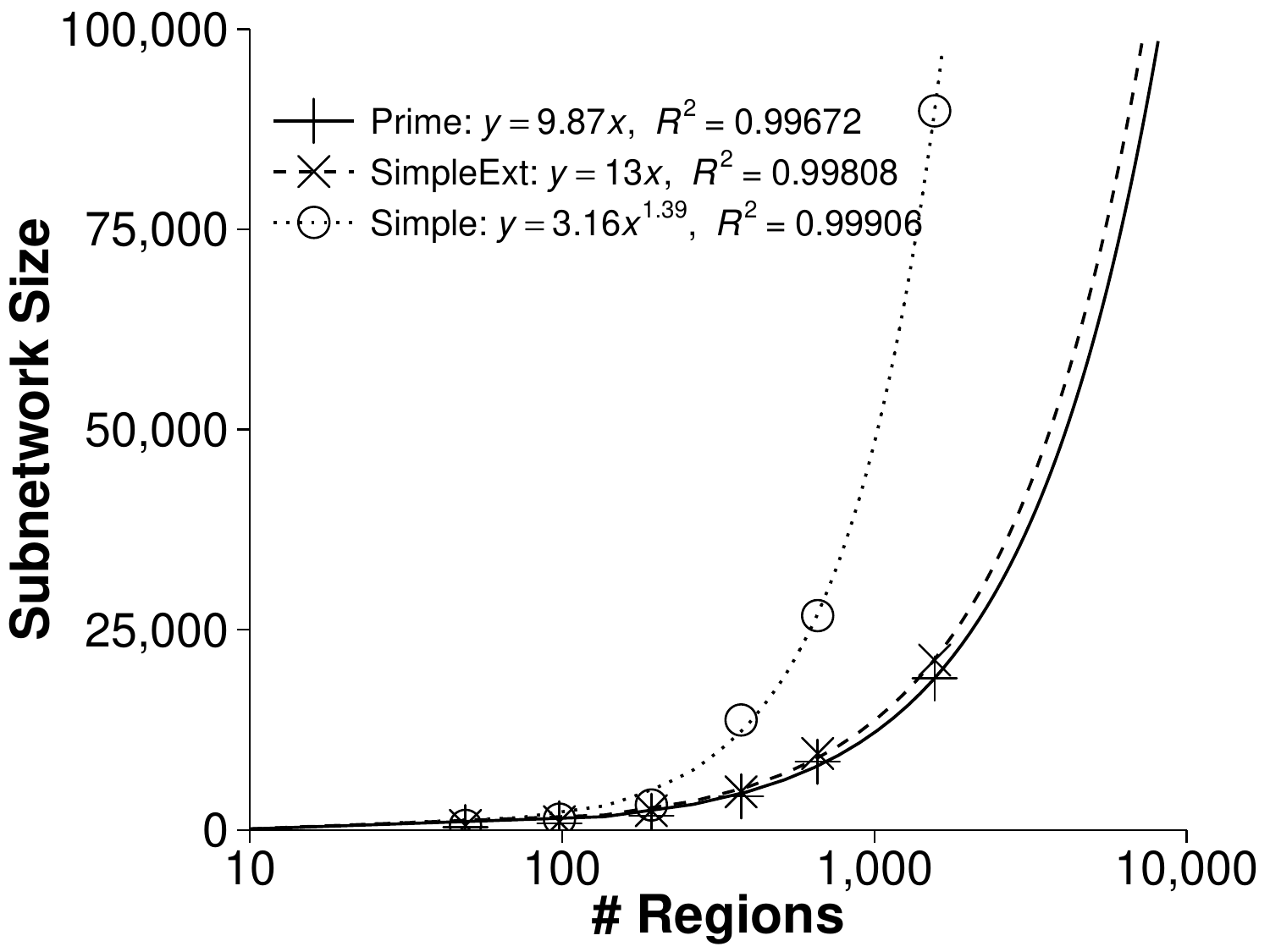}\\
\footnotesize{a.~Footprint dataset} & \footnotesize{b.~Statistical areas dataset.}\\
\end{tabular}}
\caption{{Size of subnetworks (prime, SimpleExt, and Simple) across subsets of a.~footprint and b.~statistical areas data set.}}\label{fig:RegionsVsRedundant}
\end{figure}

{
Further examination reveals that one feature that explains many of the observed differences in results is the differing proportions of \bpo\ relations in the data sets (see Table \ref{tab:constraint_proportion}). Larger proportions of \bpo\ relations are strongly related to fewer redundant relations being identified across all types of subnetwork, since \bpo\ relations typically provide limited reasoning power. Figure \ref{fig:po_redundant} demonstrates this relationship empirically for the footprint data set. However, the prime subnetwork is consistently better at identifying many more redundant relations {than} the Simple or SimpleExt algorithms when the full constraint network contains many \bpo\ relations. 
}

\begin{figure}[htb]
\centering{
\begin{tabular}{cc}
\includegraphics[width=0.5\textwidth]{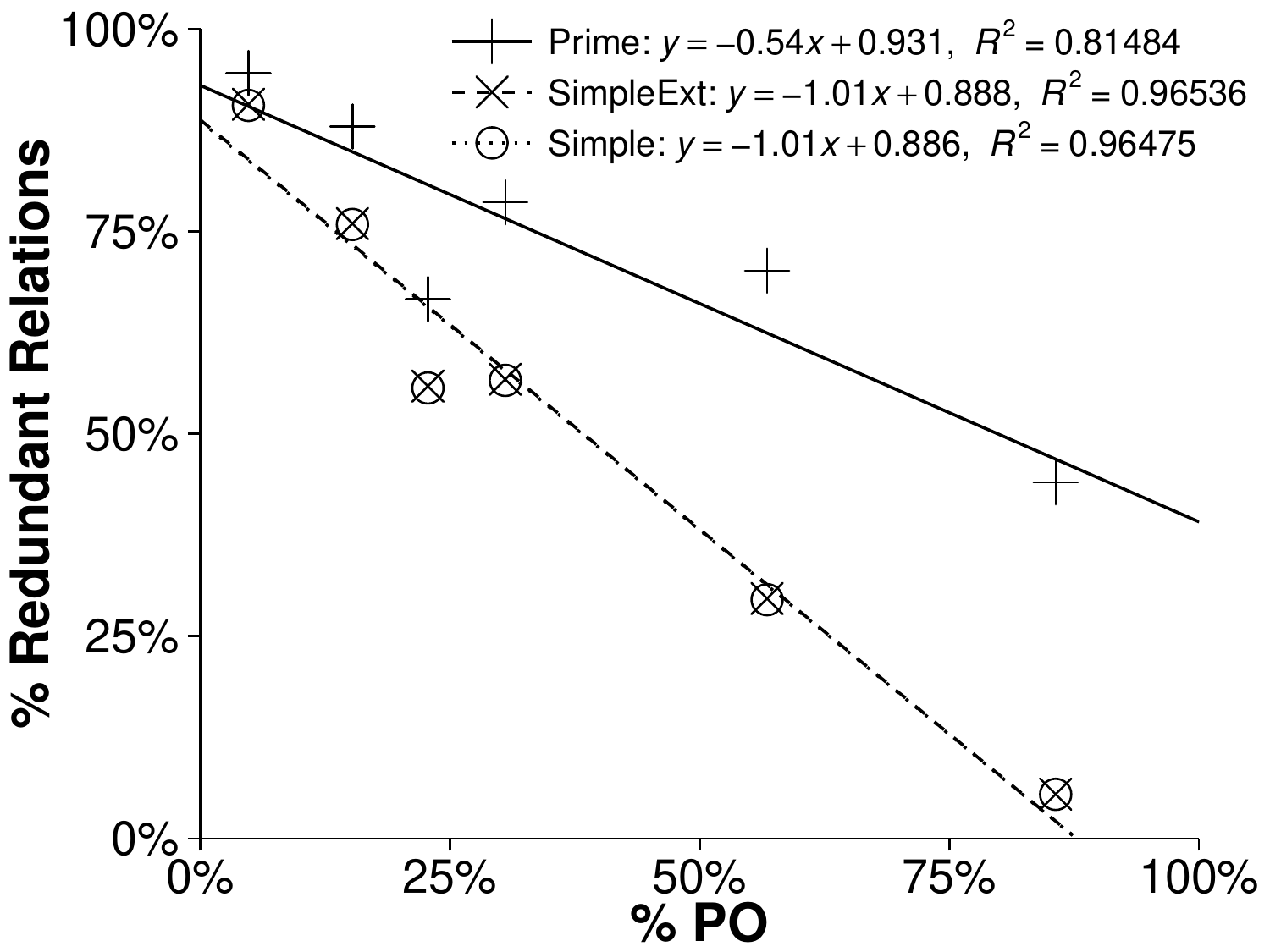}\\
\end{tabular}}
\caption{{Scatterplot of the proportion of partially overlapping relations in the full footprint constraint network, against the proportion of redundant relations identified by the prime, SimpleExt, and Simple algorithms.}}\label{fig:po_redundant}
\end{figure}

\subsection{Scalability}

{As already discussed, Algorithm 1 can compute the prime subnetwork in $O(n^3)$ {time} for any consistent network over a tractable subclass of RCC5/8, where the constraints are taken from a distributive subalgebra. Similarly, the Simple and SimpleExt algorithms must in the worst case visit all triples of regions, leading to overall $O(n^3)$ scalability. }

\begin{figure}[htb]
\centering{
\begin{tabular}{cc}
\includegraphics[width=0.47\textwidth]{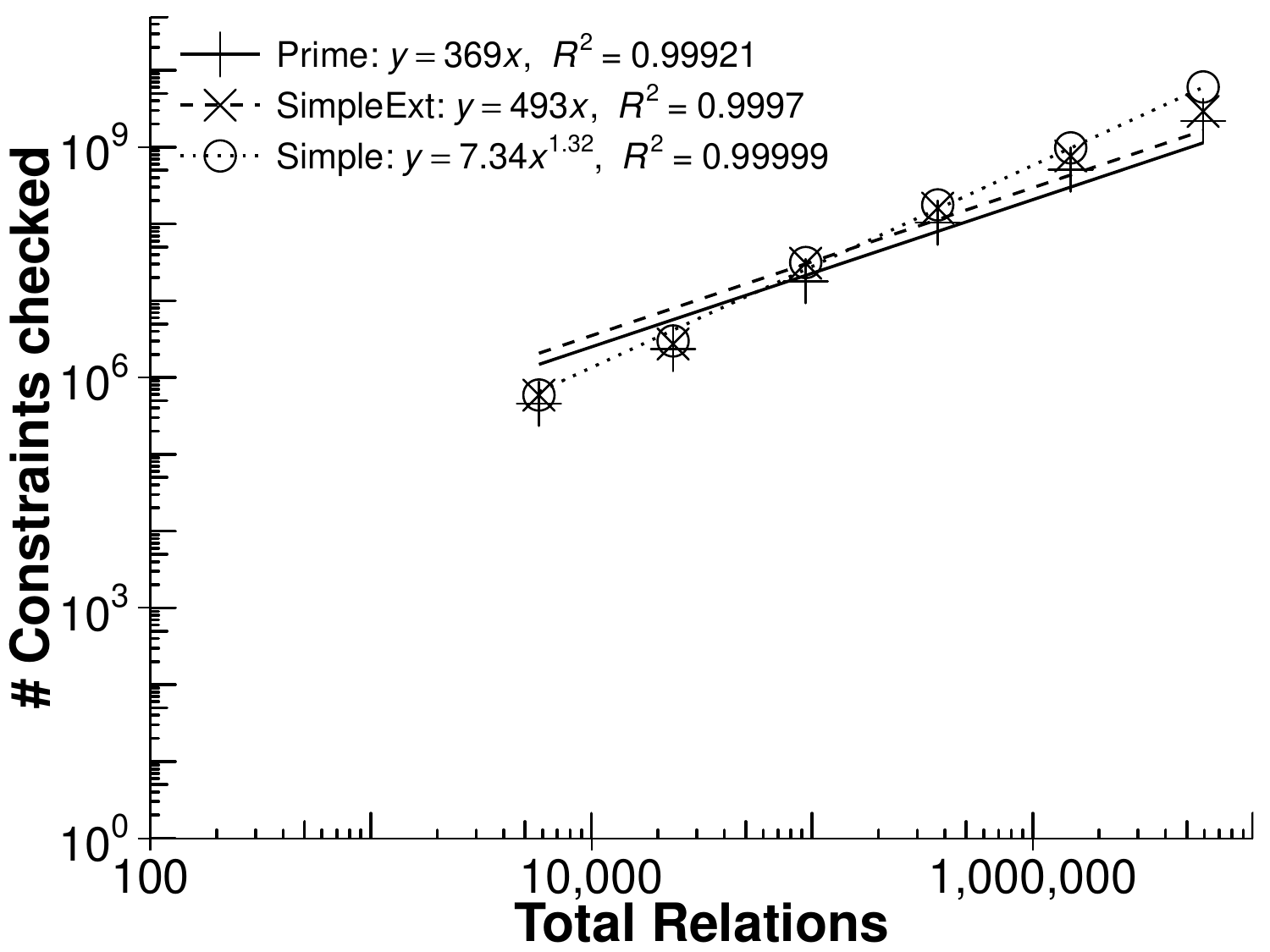}&
\includegraphics[width=0.47\textwidth]{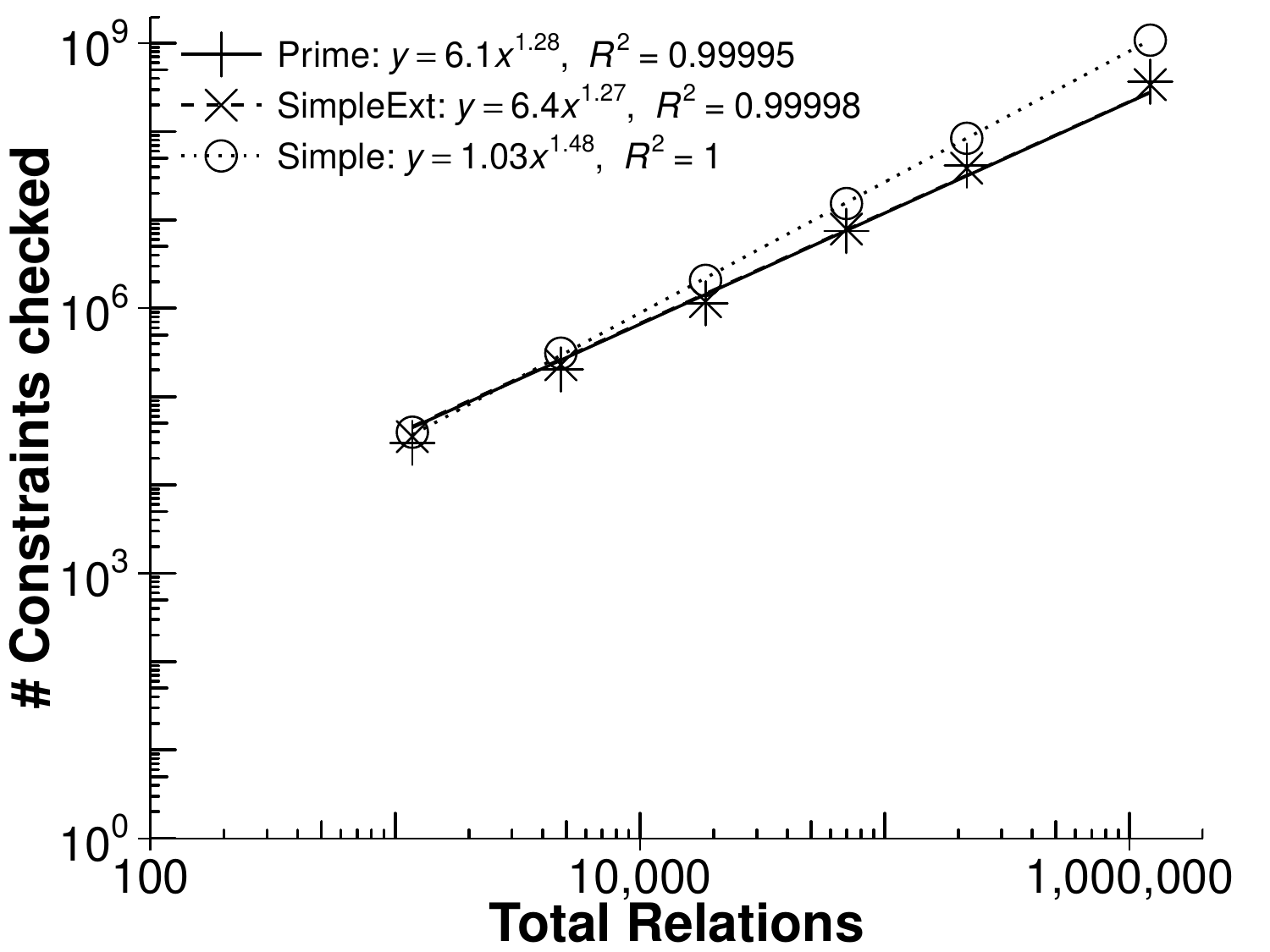}\\
\footnotesize{a.~Footprint dataset} & \footnotesize{b.~Statistical areas dataset.}\\
\end{tabular}}
\caption{{Scalability of prime, SimpleExt, and Simple algorithms, in terms of number of constraints checked across subsets of a.~footprint and b.~statistical areas data set.}}\label{fig:scalability}
\end{figure}

{On average, however, all the algorithms exhibited an average scalability in proportion to $n^2$ (where $n$ is the number of {regions}). Figure \ref{fig:scalability} shows the number of constraints checked by each algorithm, the key determining factor in computation time. All the algorithms increased linearly with the number of constraints (i.e., in proportion to $n^2$), again with the exception of the Simple algorithm {operating on} the statistical areas data set. Indeed, Algorithm 1 was on average slightly more scalable than the other two algorithms. These differences arise because on average those algorithms that are better at identifying redundant constraints are more quickly able to discard those constraints and move on to checking other constraints. } 

\subsection{{Removing Disconnected Constraints}}

{Overall, the prime subnetwork substantially reduced the number of constraints that would need to be stored to be linear in the number of regions (cf. Section \ref{sec:motivation}). Table \ref{tab:proportion} shows the proportion of constraints identified as redundant achieved by the different algorithms in the case of the full data sets, up to {98.44\%} in the case of the prime subnetwork on the highly structured statistical areas data set.\footnote{{From Table~ \ref{tab:proportion}, it can be computed that the size of the Simple subnetwork is 4.737 (1.740, resp.) times of the size of the prime subnetwork in the full statistical areas data set (the full footprint data set, resp.).}}}

\begin{table}[htb]
\centering{\begin{tabular}{r|rr}
& Footprint & Statistical areas \\\hline
Prime		&	94.58\%	&	{98.44\%}\\
SimpleExt	&	90.72\%	&	98.25\%\\
Simple		&	90.57\%	&	{92.61\%}\\
\end{tabular}}
\caption{{Proportion of constraints identified as redundant by the prime, Simple, and SimpleExt algorithms for the full footprint and statistical areas data sets.}}\label{tab:proportion}
\end{table}

{However, in some cases it might potentially {be} possible to achieve similarly high levels of storage efficiency more simply by, say, omitting the most numerous relations (typically \bdc) from the constraint network. Figure \ref{fig:DCvsPrime} shows a scatterplot of the number of constraints in the constraint network omitting \bdc\ relations, against {the number of constraints} in the prime network, both expressed as a percentage of the total number of constraints in the full constraint networks (for each of the 12 data subsets). }

\begin{figure}[htb]
\centering{
\begin{tabular}{cc}
\includegraphics[width=0.5\textwidth]{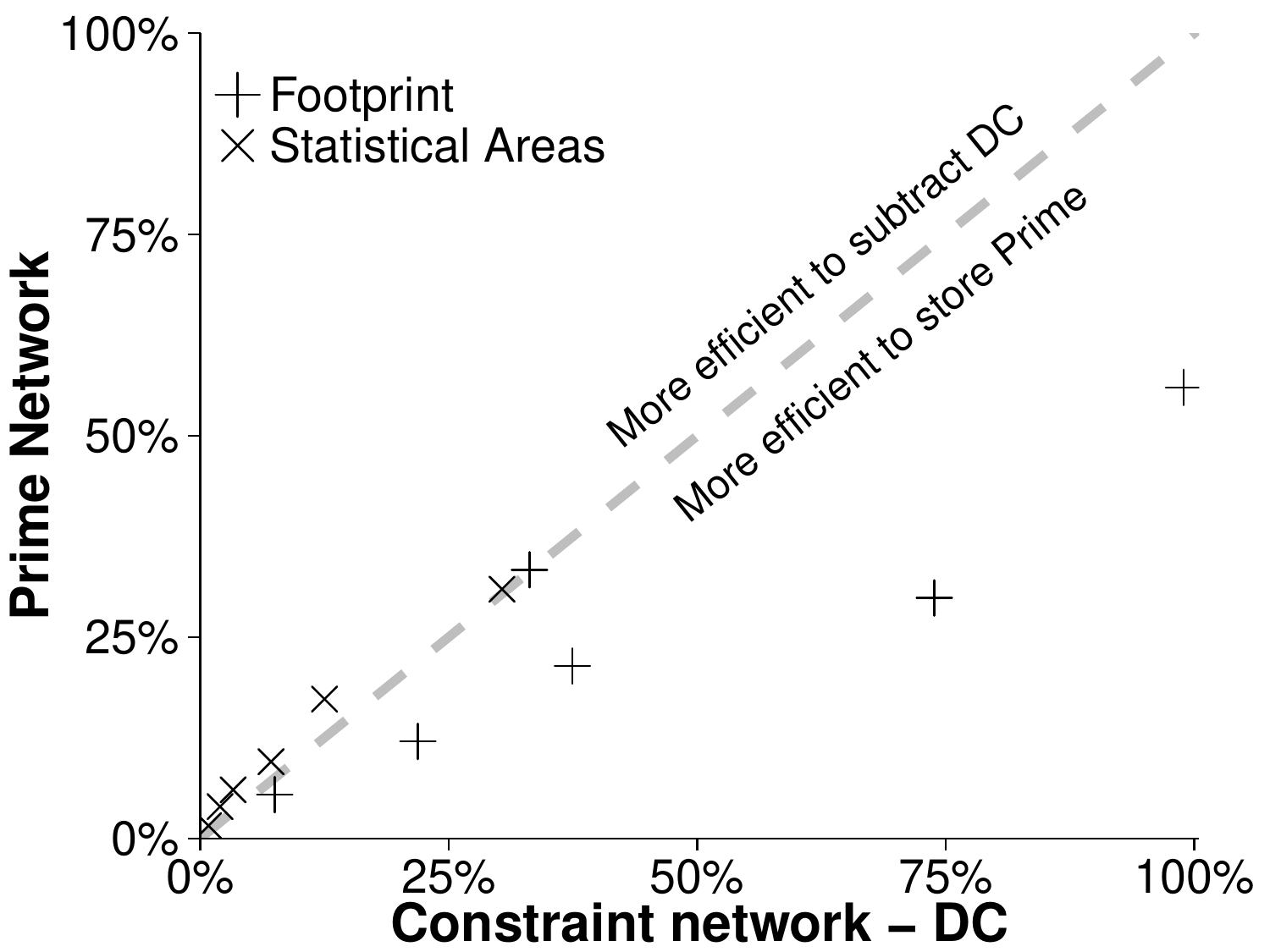}\\
\end{tabular}}
\caption{{Scatterplot of proportion of constraints in the prime network versus proportion of constraints omitting \bdc\ relations with respect to the full constraint network. }}\label{fig:DCvsPrime}
\end{figure}

{The figure shows that in some cases simply storing the constraint network without \bdc\ relations can lead to slightly more constraints omitted (those {above} the diagonal in Figure \ref{fig:DCvsPrime}). In particular, in the statistical areas data set, where the vast majority of relations are \bdc, the number of \bdc\ relations can exceed the number of redundant relations identified by the prime subnetwork. However, in most cases for the less structured footprint data set, the prime subnetwork contains substantially fewer constraints than could be achieved by simply omitting \bdc\ relations (those {below} the diagonal in Figure \ref{fig:DCvsPrime}).  Besides, while simply dropping the \bdc\ relation is competitive space-saver in some cases, it is undesirable when for example the information is incomplete and/or we cannot tell if the relation between two objects is (i) \bdc\ and dropped or (ii) missing or (iii) removed due to redundancy. }    

\subsection{{Reconstituting the Full Network}}

{As already highlighted in Section \ref{sec:motivation}, there are many potential uses for the purely qualitative prime network, without involving geometry, including facilitating the comparison and uncovering the essential structure of different constraint networks. However, one final question we address empirically is the efficiency of reconstructing the full constraint network from the prime subnetwork, when compared with doing so geometrically {if the geometric information is complete and available}. }

{The full constraint network can be reconstructed from the prime subnetwork in $O(n^3)$ time  by computing the {a-closure} of the prime subnetwork. Computing {the} constraint network directly from the geometry requires in the worst case $O(n^2)$ iterations of an $O(m^2)$ algorithm for computing the intersection between two polygons (where $m$ is the number of vertices in the polygon). In cases where $m \approx n$ this can lead to a worst case complexity of the geometric algorithm of $O(n^4)$. {We note that, in our statistical areas data set, the largest polygon contains more than 248,000 vertices, and so $m$ is indeed comparable to $n$}.}

{{However, in practice, by making use of the spatial structure of the data through algorithms (e.g., by checking for non-overlapping minimum bounding boxes for polygons before computing the polygon intersection) and spatial indexes, the geometric algorithm is expected to be on average significantly more efficient.} Figure \ref{fig:reconstitute} compares the scalability of the two approaches, {a-closure} and efficient geometric computation in an indexed spatial database. At least for the smaller data sets tested, computing the {a-closure} is significantly more efficient. For example, in the case of the smallest statistical areas data subset, computing the {a-closure} requires less than 1000th of the time of the geometric computation. However, the figure shows that using the spatial database is significantly more scalable (average-case $O(n)$ time complexity) when compared with the {a-closure} (average case approaching $O(n^3)$ complexity). }

\begin{figure}[htb]
\centering{
\begin{tabular}{cc}
\includegraphics[width=0.47\textwidth]{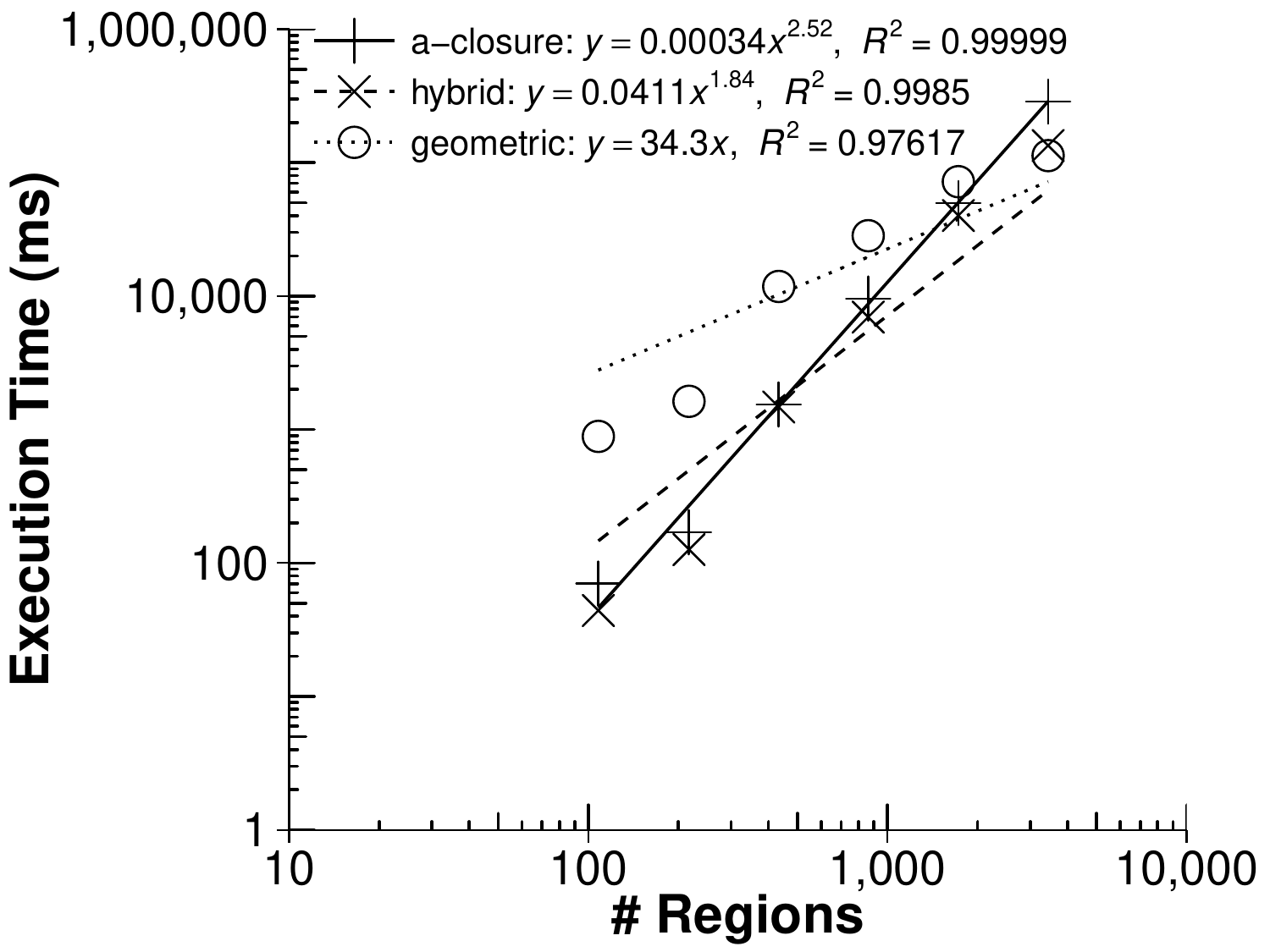}&
\includegraphics[width=0.47\textwidth]{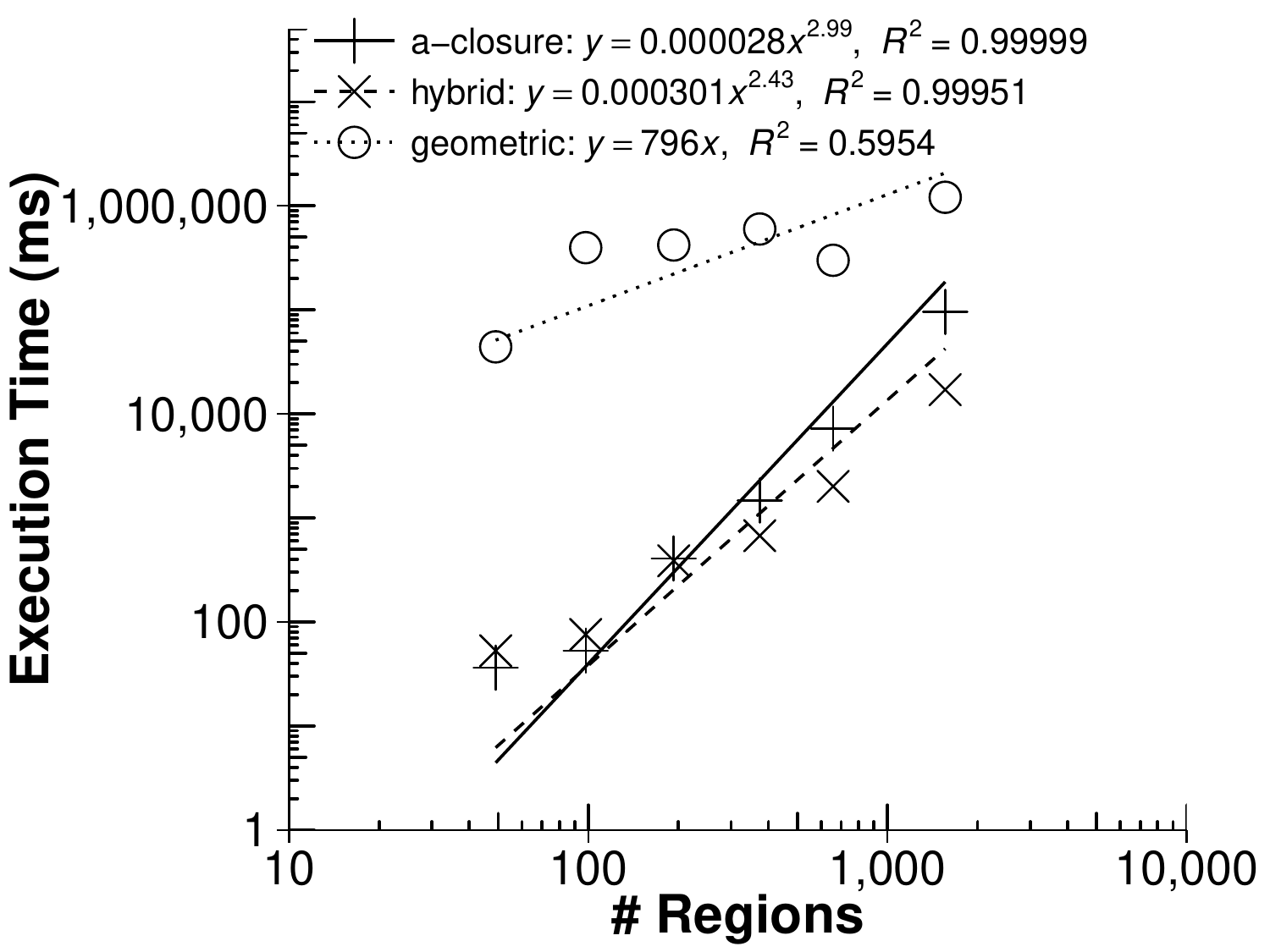}\\
\footnotesize{a.~Footprint dataset} & \footnotesize{b.~Statistical areas dataset.}\\
\end{tabular}}
\caption{{Scalability of reconstituting full constraint network using {a-closure}, efficient geometric computation in a spatial database, and our hybrid algorithm for the footprint and statistical areas data sets.} }\label{fig:reconstitute}
\end{figure}

{Figure \ref{fig:reconstitute} also shows the results of a hybrid reconstitution algorithm, that uses both the geometry and the {a-closure}. The hybrid algorithm first adds any \bdc\ relations to the prime network that can be inferred simply through comparison of the minimum bounding box of the polygon geometry (since non-intersecting minimum bounding boxes for two polygons imply a \bdc\ relations). Then the {a-closure} is computed with this partially reconstituted subnetwork. The results show a significant improvement in scalability using this approach, reducing the average case time complexity to below $O(n^2)$ in the case of the footprint dataset. Ongoing work is currently investigating further mechanisms for combining both these geometric and the qualitative aspects of regions in efficient database storage and queries. }

\subsection{{Summary}} 

{In summary, our analysis of the performance of the three subnetworks on practical geographic data sets containing thousands {of} regions demonstrates:} 

{\begin{enumerate}
\item The prime subnetwork consistently outperforms the Simple and SimpleExt algorithms in terms of the number of redundant relations identified, in particular in cases where the proportion of \bpo\ relations in the full constraint network is higher. 
\item The average case scalability for computing the prime subnetwork required $n^2$ operations, where $n$ is the number of regions. Because the prime subnetwork identified more redundant relations, it performs on average fewer constraint checks than the Simple and SimpleExt algorithms, and was in our tests always more scalable. 
\item {For less structured datasets, the} prime subnetwork can substantially improve on the number of relations identified as redundant, when compared with the naive solution of omitting \bdc\ relations from the full constraint network.
\end{enumerate}}

\section{Conclusion}

In this paper, we have systematically investigated the computational complexity of redundancy checking for RCC5/8 constraints. Although it is in general intractable, we have shown that a prime subnetwork can be found in $O(n^5)$ time for any consistent network over a tractable subclass of RCC5/8. {If the constraints are taken from a distributive subalgebra, we proved that the constraint network has a unique prime subnetwork, which can be found in cubic time. As a byproduct, we also proved that any path-consistent network over a distributive subalgebra is weakly globally consistent and minimal.}

{Our empirical analysis showed that for real geographic data sets the prime subnetwork can lead to significant increases in the number of redundant relations identified when compared with the approximations proposed by  \cite{Wallgrun12}. In practice, the algorithm was efficient, exhibiting average case $O(n^2)$ scalability. The redundant relations identified by the prime subnetwork can also significantly outnumber \bdc\ relations, especially in less structured geographic data sets that may contain a significant minority of \bpo\ relations. 
}

It is worth noting that {a large part of our results} can be applied to several other qualitative calculi (like PA, IA, CRA, and RA) immediately, but Propositions~\ref{prop:redun} and ~\ref{prop:redun+} and Theorem~\ref{thm:core=Gamma+} do use the particular algebraic properties of RCC5/8 {(see Table~\ref{tab:results} for a summary).} For example, we have an all-different and path-consistent basic IA network which is not equivalent to its core.  Future work will consider how to extend our results to {IA, RA and other calculi}.  

\begin{table}[h]
 \centering
\begin{tabular}{c|l|l|l|l}
   & PA & IA & CRA & RA \\ \hline
  Proposition~\ref{prop:ra}& + & + & + & + \\ 
  Proposition~\ref{wgcISmin}& + & + & + & + \\ 
  Lemma~\ref{lemma:distri}& + & + & + & + \\ 
  Lemma~\ref{lem:3x}& + & + & + & + \\ 
  Proposition~\ref{prop:redundant-co-npc}& +s & + & + & + \\ 
  Proposition~\ref{prop:redundant-horn-5}& + & + & + & +p \\ 
  Proposition~\ref{prop:alldiff}& + & + & + & +p \\ 
  Theorem~\ref{thm:wglobal}& + & + & + & +p\\ 
  Theorem~\ref{thm:minimal}& + & + & + & +p\\ 
  Lemma~\ref{lem:int-paths}& + & + & + & +\\ 
  Lemma~\ref{prop:+circle+}& + & + & + & +p\\ 
  Lemma~\ref{cor:intesection-redundant-if}& + & + & + & +\\ 
  Lemma~\ref{cor:intesection-redundant-onlyif}& + & + & + & +p\\ 
  Lemma~\ref{lem:pcnetwork-redun}& + & + & + & +p\\ 
  Lemma~\ref{prop:redun-onlyif}& + & + & + & +\\ 
  Proposition~\ref{prop:redun}& + & - & + & -\\ 
  Proposition~\ref{prop:redun+}& + & - & + & -\\ 
  Theorem~\ref{thm:core=Gamma+}& + & - & + & -\\ 
 \end{tabular}
\hfill
\caption{{Applicability of the results in this paper to other calculi, where + (-) indicates the corresponding result holds  (does not hold) for that calculus, +s indicates that for PA it is tractable to decide if a constraint is redundant, and +p indicates that the result holds for any tractable subclass of RA over which path-consistency implies consistency.}}\label{tab:results}
 \end{table}

{We are also developing further the practical applications of prime subnetworks. In addition to pursuing a more systematic exploration of the applications to saving storage than {the one given in Section 1.1}, current work is investigating other aspects of prime subnetworks, in particular the structure and comparison of different prime subnetworks of sets of footprints. }

\section*{Acknowledgments}

The authors are grateful to Ross Purves, University of Z\"urich, for helping to {develop the} footprint data set illustrated in Figure~\ref{fig:southampton}.  This work was partially supported by ARC (DP120103758, DP120104159) and NSFC (61228305).   

\appendix


\section{Maximal Distributive Subalgebras of RCC5/8} \label{app:dis}
{A distributive subalgebra $\mathcal{S}$ is \emph{maximal} if there is no other distributive subalgebra that properly contains $\mathcal{S}$. To compute the maximal distributive subalgebras, we first compute {$\clb_l$,} the closure of $\mathcal{B}_l$ in RCC$l$ under converse, weak composition, and intersection, and then check by a program if $\clb_l\cup Z$ satisfies distributivity for some subset $Z$ of RCC$l$.}  

{
Write $\mathcal{D}_l$ for the set of RCC$l$ relations $\alpha$ such that $\clb_l\cup\{\alpha\}$ satisfies distributivity. We then check for every pair of relations $\alpha,\beta$ in $\mathcal{D}_l$ if $\clb_l\cup\{\alpha,\beta\}$ satisfies distributivity. If this is the case, then we say  $\alpha$ has d-relation to $\beta$. Fortunately, the result shows that there are precisely two disjoint subsets $X_l$ and $Y_l$ (which form a partition of $\mathcal{D}_l$) such that each relation $\alpha$ in $X_l$ ($Y_l$, respectively) has d-relation to every other relation in $X_l$ ($Y_l$, respectively), but has no d-relation to any relation in $Y_l$ ($X_l$, respectively).  Moreover, $\clb_l\cup X_l$ and $\clb_l\cup Y_l$ are both distributive subalgebras of RCC$l$. It is clear that these are the only maximal distributive subalgebras of RCC$l$. 
}

For RCC5, the closure of basic relations $\clb_5$ contains 12 nonempty relations. These are the five basic relations, and the following 7 relations (cf. Section 2.4.)
\begin{align*}
&\{\bpo,\bpp\},\{\bpo,\bppi\},\{\bpo,\bpp,\bppi,\beq\},\\
&\{\bdr,\bpo,\bpp\},\{\bdr,\bpo,\bppi\},\{\bdr,\bpo\},\star_5.
\end{align*}

The first maximal distributive subalgebra, denoted by $\mathcal{D}^5_{14}$, contains (except relations in $\clb_5$) 
$$\{\bpp,\beq\},\{\bppi,\beq\}.$$
The second maximal distributive subalgebra, denoted by $\mathcal{D}^5_{20}$, contains in addition the following eight relations
\begin{align*}
&\{\bpo,\beq\},\{\bpo,\bpp,\beq\},\{\bpo,\bpp,\bppi\},\{\bpo,\bppi,\beq\},\\
&\{\bdr,\bpo,\bpp,\bppi\},\{\bdr,\bpo,\bppi,\beq\},\\
&\{\bdr,\bpo,\beq\},\{\bdr,\bpo,\bpp,\beq\}.
\end{align*}
It is easy to see that both $\mathcal{D}^5_{14}$ and $\mathcal{D}^5_{20}$ are contained in $\horn_5$, the maximal tractable subclass of RCC5 identified in \cite{RenzN97,Jonsson97}.

For RCC8, the closure of basic relations contains 37 nonempty relations. These are the eight basic relations and the following 29 relations 
\begin{align*}
&\{\bpo,\btpp\},\{\bpo,\btppi\},\{\bpo,\btpp,\bntpp\},\\
&\{\bpo,\btppi,\bntppi\},\{\bpo,\btpp,\btppi,\beq\},\\
&\{\bpo,\btpp,\bntpp,\btppi,\beq\},\\
&\{\bpo,\btpp,\btppi,\bntppi,\beq\},\\
&\{\bpo,\btpp,\bntpp,\btppi,\bntppi,\beq\},\\
&\{\btpp,\bntpp\},\{\btppi,\bntppi\},\\
&\{\bec,\bpo\},\{\bec,\bpo,\btpp\},\{\bec,\bpo,\btppi\},\\
&\{\bec,\bpo,\btpp,\bntpp\},\\
&\{\bec,\bpo,\btppi,\bntppi\},\\
&\{\bec,\bpo,\btpp,\btppi,\beq\},\\
&\{\bec,\bpo,\btpp,\bntpp,\btppi,\beq\},\\
&\{\bec,\bpo,\btpp,\btppi,\bntppi,\beq\},\\
&\{\bec,\bpo,\btpp,\bntpp,\btppi,\bntppi,\beq\},\\
&\{\bdc,\bec\},\{\bdc,\bec,\bpo\},\{\bdc,\bec,\bpo,\btpp\},\\
&\{\bdc,\bec,\bpo,\btppi\},\\
&\{\bdc,\bec,\bpo,\btpp,\bntpp\},\\
&\{\bdc,\bec,\bpo,\btppi,\bntppi\},\\
&\{\bdc,\bec,\bpo,\btpp,\btppi,\beq\},\\
&\{\bdc,\bec,\bpo,\btpp,\bntpp,\btppi,\beq\},\\
&\{\bdc,\bec,\bpo,\btpp,\btppi,\bntppi,\beq\},\star_8,
\end{align*}
where $\star_8$ is the universal relation consisting of all RCC8 basic relations.

The first maximal distributive subalgebra, denoted by $\mathcal{D}^8_{41}$, contains in addition the following four relations
\begin{align*}
&\{\btpp,\beq\},\{\btpp,\bntpp,\beq\},\\
&\{\btppi,\beq\},\{\btppi,\bntppi,\beq\}.
\end{align*}
This distributive subalgebra turns out to be exactly the class of convex RCC8 relations identified in \cite{ChandraP05}. The second maximal distributive subalgebra, denoted by $\mathcal{D}^8_{64}$, contains in addition the following 27 relations
\begin{align*}
& \{\bpo,\beq\},\{\bpo,\btpp,\beq\},\\
&\{\bpo,\btppi,\beq\},\{\bpo,\btpp,\btppi\},\\
&\{\bpo,\btpp,\bntpp,\beq\},\\
&\{\bpo,\btppi,\bntppi,\beq\},\\
&\{\bpo,\btpp,\btppi,\bntppi\},\\
&\{\bpo,\btpp,\bntpp,\btppi\},\\
&\{\bpo,\btpp,\bntpp,\btppi,\bntppi\},\\
&\{\bec,\bpo,\beq\},\{\bec,\bpo,\btpp,\beq\},\\
&\{\bec,\bpo,\btppi,\beq\},\\
&\{\bec,\bpo,\btppi,\bntppi,\beq\},\\
&\{\bec,\bpo,\btpp,\bntpp,\beq\},\\
&\{\bec,\bpo,\btpp,\btppi\},\\
&\{\bec,\bpo,\btpp,\btppi,\bntppi\},\\
&\{\bec,\bpo,\btpp,\bntpp,\btppi\},\\
&\{\bec,\bpo,\btpp,\bntpp,\btppi,\bntppi\},\\
&\{\bdc,\bec,\bpo,\beq\},\{\bdc,\bec,\bpo,\btpp,\beq\},\\
&\{\bdc,\bec,\bpo,\btppi,\beq\},\\
&\{\bdc,\bec,\bpo,\btpp,\btppi\},\\
&\{\bdc,\bec,\bpo,\btppi,\bntppi,\beq\},\\
&\{\bdc,\bec,\bpo,\btpp,\bntpp,\beq\},\\
&\{\bdc,\bec,\bpo,\btpp,\bntpp,\btppi\},\\
&\{\bdc,\bec,\bpo,\btpp,\btppi,\bntppi\},\\
&\{\bdc,\bec,\bpo,\btpp,\bntpp,\btppi,\bntppi\}.
\end{align*}
It is easy to check that both $\mathcal{D}^8_{41}$ and $\mathcal{D}^8_{64}$ are contained in $\hornr$, one of the three maximal subclasses of RCC8 identified in \cite{Renz99}.

\bibliographystyle{plain}

\bibliography{qsr2014}
\end{document}